\documentclass[12pt,letterpaper]{article}
\usepackage{color}
\usepackage{graphicx}
\usepackage{epsfig,subfigure}
\usepackage{amssymb}
\usepackage{amsthm}
\usepackage{amsmath}
\usepackage[margin=1in]{geometry}
\usepackage{float}

\usepackage{algorithm}
\usepackage{algorithmic}
\usepackage{epstopdf}

\usepackage{hyperref}

\usepackage{bbm,epsfig}
\usepackage{dsfont}
\usepackage{verbatim}
\usepackage{url}

\newtheorem{theorem}{Theorem}

\newtheorem{lemma}[theorem]{Lemma}

\newtheorem*{theorem*}{Theorem}
\newtheorem*{proposition*}{Proposition}
\newtheorem*{lemma*}{Lemma}
\newtheorem*{corollary*}{Corollary}
\newtheorem*{claim*}{Claim}

\newcommand{\ip}[2]{\langle#1,#2\rangle}

\newcommand{\bx}{\mathbf{x}}
\newcommand{\bq}{\mathbf{q}}
\newcommand{\by}{\mathbf{y}}
\newcommand{\bv}{\mathbf{v}}
\newcommand{\bw}{\mathbf{w}}
\newcommand{\ba}{\mathbf{a}}
\newcommand{\bb}{\mathbf{b}}
\newcommand{\bz}{\mathbf{z}}
\newcommand{\bu}{\mathbf{u}}
\newcommand{\bs}{\mathbf{s}}
\newcommand{\bg}{\mathbf{g}}
\newcommand{\dist}{\text{dist}}
\newcommand{\Qt}{Q_{\tau}}
\newcommand{\Qtt}{Q_{\tau+1}}
\newcommand{\Ftt}{F_{\tau+1}}

\newcommand{\Rtt}{R_{\tau+1}}

\newcommand{\tv}{\tilde{\bv}}

\newcommand{\N}{\mathcal{N}}

\newcommand{\tr}{\mathop{\mathrm{Tr}}}

\title{Memory Limited, Streaming PCA}
\author{Ioannis Mitliagkas$^t$, Constantine Caramanis$^t$, Prateek Jain$^m$ \thanks{Email: \{ioannis,constantine\}@utexas.edu;prajain@microsoft.com}
\vspace{0.25in} \\ $^t$ The University of Texas at Austin  \vspace{0.1in} \\ $^m$ Microsoft Research India, Bangalore}

\begin{document}

\maketitle

\begin{abstract} 
%Memory-Complexity of High-Dimensional Streaming PCA

We consider streaming, one-pass principal component analysis (PCA), in the high-dimensional regime, with limited memory. Here, $p$-dimensional samples are presented sequentially, and the goal is to produce the $k$-dimensional subspace that best approximates these points. Standard algorithms require $O(p^2)$ memory; meanwhile no algorithm can do better than $O(kp)$ memory, since this is what the output itself requires. Memory (or storage) complexity is most meaningful when understood in the context of computational and sample complexity. Sample complexity for high-dimensional PCA is typically studied in the setting of the {\em spiked covariance model}, where $p$-dimensional points are generated from a population covariance equal to the identity (white noise) plus a low-dimensional perturbation (the spike) which is the signal to be recovered. It is now well-understood that the spike can be recovered when the number of samples, $n$, scales proportionally with the dimension, $p$. Yet, all algorithms that provably achieve this, have memory complexity $O(p^2)$. Meanwhile, algorithms with memory-complexity $O(kp)$ do not have provable bounds on sample complexity comparable to $p$. We present an algorithm that achieves both: it uses $O(kp)$ memory (meaning storage of any kind) and is able to compute the $k$-dimensional spike with $O(p \log p)$ sample-complexity -- the first algorithm of its kind. While our theoretical analysis focuses on the spiked covariance model, our simulations show that our algorithm is successful on much more general models for the data.

\end{abstract}

\section{Introduction}
Principal component analysis is a fundamental tool for dimensionality reduction, clustering, classification, and many more learning tasks. It is a basic preprocessing step for learning, recognition, and estimation procedures. The core computational element of PCA is performing a (partial) singular value decomposition, and much work over the last half century has focused on efficient algorithms (e.g., \cite{golub_matrix_2012} and references therein) and hence on {\em computational complexity}. 

%High-dimensional data applications, such as face recognition (\cite{phillips_feret_2000}), document vectors in latent semantic analysis (\cite{deerwester_indexing_1990})}, sound recordings, biometrics, video, etc., have drawn the focus to the high-dimensional setting, and the problem of understanding {\em sample complexity} of covariance estimation. Work initiated in \cite{Johnstone2001} and significant work since then (see \cite{vershynin_how_2010} and references therein) has explored the power of batch PCA in the $p$-dimensional setting with sub-Gaussian noise, and demonstrated that the singular value decomposition (SVD) of the empirical covariance matrix succeeds in recovering the principal components (extreme eigenvectors of the population covariance) with high probability, given $n = O(p)$ samples. 

The recent focus on understanding high-dimensional data, where the dimensionality of the data scales together with the number of available sample points, has led to an exploration of the {\it sample complexity} of covariance estimation. This direction was largely influenced by Johnstone's {\em spiked covariance model}, where data samples are drawn from a distribution whose (population) covariance is a low-rank perturbation of the identity matrix \cite{johnstone_distribution_2001}. Work initiated there, and also work done in \cite{vershynin_how_2010} (and references therein) has explored the power of batch PCA in the $p$-dimensional setting with sub-Gaussian noise, and demonstrated that the singular value decomposition (SVD) of the empirical covariance matrix succeeds in recovering the principal components (extreme eigenvectors of the population covariance) with high probability, given $n = O(p)$ samples. 

This paper brings the focus on another critical quantity: memory/storage. This is relevant in the so-called {\em streaming data model}, where the samples $\mathbf{x}_t \in \mathbb{R}^p$ are collected sequentially, and unless we store them, they are irretrievably gone.\footnote{This is similar to what is sometimes referred to as the {\em single pass} model.} The only currently available algorithms with provable sample complexity guarantees either store all $n = O(p)$ samples (note that for more than a single pass over the data, the samples must all be stored) or explicitly form the empirical $p \times p$ (typically dense) covariance matrix. Either case requires at least $O(p^2)$ storage. Despite the availability of massive local and distributed storage systems, for high-dimensional applications (e.g., where data points are high resolution photographs, biometrics, video, etc.), $p$ could be on the order of $10^{10} - 10^{12}$, making $O(p^2)$ prohibitive, if not in fact impossible to manage. Indeed, at multiple computing scales, manipulating vectors of length $O(p)$ is possible, when storage of $O(p^2)$ is not. A typical desktop may have 10-20 GB of RAM, but will not have more than a few TB of total storage. A modern smart-phone may have as much as a GB of RAM, but has a few GB, not TB, of storage.

% The only algorithms with provable sample complexity bounds fall into one of these two categories. Yet $O(p^2)$ storage may be prohibitive in several important high-dimensional applications; for instance, if the data points are high resolution photographs {\color{black} in face recognition applications (\cite{phillips_feret_2000}), document vectors in latent semantic analysis (\cite{deerwester_indexing_1990})}, sound recordings, biometrics, or video, then $p$ could be on the order of $10^{10} - 10^{12}$. %And while powerful computers often manipulate such files, they are often collected by much less well-endowed counterparts, such as smart phones and tablets. 
%Several practical applications have to handle high-dimensional data; data points could be photographs, sound recordings, biometrics, or video. And while powerful computers often manipulate such files, they are often collected by much less well-endowed counterparts, such as smart phones and tablets.
%Large-scale storage systems offer the promise of essentially limitless storage and distributed computing power, however, for many real systems, the quadratic storage requirement is simply not an option. A typical desktop may have 10-20 GB of RAM, but will not have more than a few TB of total storage. A modern smart-phone may have as much as a GB of RAM, but has a few GB, not TB, of storage. That is, at multiple computing scales, manipulating vectors of length $O(p)$ is possible, when storage of $O(p^2)$ is not. 

We consider the {\em streaming data} setting, where data points are generated sequentially, and are never stored. In the setting of the so-called spiked covariance model (and natural generalizations) we show that a simple algorithm requiring $O(kp)$ storage -- the best possible -- performs as well as batch algorithms (namely, SVD on the empirical covariance matrix), with sample complexity $O(p \log p)$. To the best of our knowledge, this is the only algorithm with both storage complexity and sample complexity guarantees. We discuss the connection to past work in more detail in Section \ref{sec:related}. We introduce the model with all related details in Section \ref{sec:formulation}, and present the solution to the rank 1 case, the rank $k$ case, and the perturbed-rank-$k$ case in Sections \ref{sec:rank1}, \ref{sec:rankk} and \ref{sec:robust}, respectively. In Section \ref{sec:experiments} we provide simulations that not only confirm the theoretical results, but demonstrate that our algorithm works well outside the assumptions of our main theorems.

\section{Related Work}
\label{sec:related}
Memory- and computation-efficient algorithms that operate on streaming data are plentiful in the literature and many seem to do well in practice. However, there is no algorithm that provably recovers the principal components in the same noise and sample-complexity regime as the batch PCA algorithm does \emph{and} maintains a provably light memory footprint. Because of the practical relevance, there has been renewed interest recently in this problem, and the fact that this is an important unresolved issue has been pointed out in numerous places, e.g., \cite{warmuth_randomized_2008,arora_stochastic_2012}. %There has been more work than possible to list; we discuss the most relevant work here. 

A large body of work has focused on the non-statistical data paradigm that deals with a fixed pool of samples. This includes work on online PCA and low-rank matrix approximation in the streaming scenario, including sketching and dimensionality-reduction based techniques.

Online-PCA for {\em regret minimization} has been considered in several papers, most recently in \cite{warmuth_randomized_2008}, where the multiplicative weights approach is adapted for this problem (now experts correspond to subspaces). The goal there is to control the regret, improving on the natural follow-the-leader algorithm that performs batch-PCA at each step. However, the algorithm can require $O(p^2)$ memory, in order to store the multiplicative weights. A memory-light variant described in \cite{arora_stochastic_2012} typically requires much less memory, but there are no guarantees for this, and moreover, for certain problem instances, its memory requirement is on the order of $p^2$. 
%A related approach is to consider the dual of the dual of PCA, and use a gradient-based algorithm. While convexity now guarantees convergence, the optimization variable is now a $p \times p$ matrix, and there is no guarantee that the iterative procedure will maintain a low-rank matrix (\cite{srebro_personal}). 

Sub-sampling, dimensionality-reduction and sketching form another family of low-complexity and low-memory techniques, see, e.g., \cite{clarkson_numerical_2009,nadler_finite_2008,halko_finding_2011}. These save on memory and computation by performing SVD on the resulting smaller matrix. The results in this line of work provide worst-case guarantees over the pool of data, and typically require a rapidly decaying spectrum (which we do not have in our setting) to produce good bounds. More fundamentally, these approaches are not appropriate for data coming from a statistical model such as the spiked covariance model. It is clear that subsampling approaches, for instance, simply correspond to discarding most of the data, and for fundamental sample complexity reasons, cannot work. Sketching produces a similar effect: each column of the sketch is a random ($+/-$) sum of the data points. If the data points are, e.g., independent Gaussian vectors, then so will each element of the sketch, and thus this approach again runs against fundamental sample complexity constraints. Indeed, it is straightforward to check that the guarantees presented in (\cite{clarkson_numerical_2009,halko_finding_2011}) are not strong enough to guarantee recovery of the spike. This is not because the results are weak; it is because they geared towards worst-case bounds.

Algorithms focused on sequential SVD (e.g., \cite{brand_incremental_2002,brand_fast_2006}, \cite{comon_tracking_1990},\cite{li_incremental_2004} and more recently \cite{balzano_online_2010,he_online_2011}) seek to have the best subspace estimate at every time (i.e., each time a new data sample arrives) but without performing full-blown SVD at each step. While these algorithms indeed reduce both the computational and memory burden of batch-PCA, there are no  rigorous guarantees on the quality of the principal components or on the {\em statistical performance} of these methods.
%More recent approaches \cite{balzano_online_2010,he_online_2011} obviate the need for re-orthogonalization by taking into consideration the geodesics of the Grassmanian manifold, but also fail to come up with sample complexity bounds, or even guarantees of universal consistency.

%
%Then we go through work relating to the statistical data paradigm we are mostly interested in.
%In particular, we discuss the relationship to EM-based methods, online-PCA, incremental models for SVD and PCA and stochastic-approximation-based methods.
%

In a Bayesian mindset, some researchers have come up with expectation maximization approaches \cite{roweis_em_1998,tipping_probabilistic_1999}, that can be used in an incremental fashion. The finite sample behavior is not known. 

Stochastic-approximation-based algorithms along the lines of \cite{robbins_stochastic_1951} are also quite popular, because of their low computational and memory complexity, and excellent performance in practice. They go under a variety of names, including {\em Incremental PCA} (though the term {\em Incremental} has been used in the online setting as well \cite{herbster_tracking_2001}), Hebbian learning, and stochastic power method \cite{arora_stochastic_2012}. The basic algorithms are some version of the following: upon receiving data point $\mathbf{x}_t$ at time $t$, update the estimate of the top $k$ principal components via:
\begin{equation}
\label{eq:SA}
U^{(t+1)} = {\rm Proj}(U^{(t)} + \eta_t \mathbf{x}_t \mathbf{x}_t^{\top} U^{(t)}),
\end{equation}
where ${\rm Proj(\cdot)}$ denotes the ``projection'' that takes the SVD of the argument, and sets the top $k$ singular values to $1$ and the rest to zero (see \cite{arora_stochastic_2012} for further discussion). 

While empirically these algorithms perform well, to the best of our knowledge - and efforts - there does not exist any rigorous finite sample guarantee for these algorithms. The analytical challenge seems to be the high variance at each step, which makes direct analysis difficult.

In summary, while much work has focused on memory-constrained PCA, there has as of yet been no work that simultaneously provides sample complexity guarantees competitive with batch algorithms, and also memory/storage complexity guarantees close to the minimal requirement of $O(kp)$ -- the memory required to store only the output. We present an algorithm that provably does both.

\section{Problem Formulation and Notation}
\label{sec:formulation}
We consider a streaming model, where at each time step $t$, we receive a point $\mathbf{x}_t \in \mathbb{R}^p$. Furthermore, any vector that is not explicitly stored can never be revisited. %By {\em streaming}, we mean that the data are generated and fed to the algorithm. Any vector that is not explicitly stored can never be revisited.
 Now, our goal is to compute the top $k$ principal components of the data: the $k$-dimensional subspace that offers the best squared-error estimate for the points. We assume a probabilistic generative model, from which the data is sampled at each step $t$. Specifically, we assume, 
\begin{equation}
  \label{eq:model}
  \bx_t=A \bz_t+\bw_t,
\end{equation}
where $A\in \mathbb{R}^{p\times k}$ is a fixed matrix, $\bz_t\in \mathbb{R}^{k\times 1}$ is a multivariate normal random variable, i.e., 
$$
\bz_t \sim \mathcal{N}(0_{k\times 1}, I_{k\times k}),
$$ 
and vector $\bw_t\in \mathbb{R}^{p\times 1}$ is the ``noise'' vector and is also sampled from a multivariate normal distribution, i.e., 
$$
\mathbf{w}_t\sim \mathcal{N}(0_{p\times 1}, \sigma^2 I_{p\times p}).
$$
Furthermore, we assume that all $2n$ random vectors ($\bz_t, \bw_t, \forall 1\leq t\leq n$) are mutually independent. 

 In this regime, it is well-known that batch-PCA is asymptotically consistent (hence recovering $A$ up to unitary transformations) with number of samples scaling as $n = O(p)$ \cite{vershynin_introduction_2010}. It is interesting to note that in this high-dimensional regime, the signal-to-noise ratio quickly approaches zero, as the signal, or ``elongation'' of the major axis, $\|Az\|_2$, is $O(1)$, while the noise magnitude, $\|\mathbf{w}\|_2$, scales as $O(\sqrt{p})$. The central goal of this paper is to provide finite sample guarantees for a streaming algorithm that requires memory no more than $O(k p)$ and matches the consistency results of batch PCA in the sampling regime $n = O(p)$ (possibly with additional log factors, or factors depending on $\sigma$ and $k$). 

We denote matrices by capital letters (e.g. $A$) and vectors by lower-case bold-face letters ($\bx$). $\|\bx\|_q$ denotes the $\ell_q$ norm of $\bx$; $\|\bx\|$ denotes the $\ell_2$ norm of $\bx$. $\|A\|$ or $\|A\|_2$ denotes the spectral norm of $A$ while $\|A\|_F$ denotes the Frobenius norm of $A$. Without loss of generality (WLOG), we assume that: $\|A\|_2=1$, where $\|A\|_2=\max_{\|\bx\|_2=1}\|A\bx\|_2$ denotes the spectral norm of $A$. Finally, we write $\ip{\ba}{\bb}=\ba^{\top}\bb$ for the inner product between $\ba$, $\bb$. In proofs the constant $C$ is used loosely and its value may vary from line to line.

\section{Algorithm and Guarantees}
In this section, we present our proposed algorithm and its finite sample analysis. It is a block-wise stochastic variant of the classical power-method. Stochastic versions of the power method are already popular in the literature and are known to have good empirical performance; see \cite{arora_stochastic_2012} for a nice review of such methods. However, the main impediment to the analysis of such stochastic  algorithms (as in \eqref{eq:SA}) is the potentially large variance of each step, due primarily to the high-dimensional regime we consider, and the vanishing SNR. 
%In literature, there exists several stochastic versions of the power method
%As discussed above, Indeed, many stochastic versions of the power method have been proposed before; \cite{arora_stochastic_2012} contains a good review of such methods, as well as further discussion of the challenges of analyzing their behavior. 

This motivated us to consider a modified stochastic power method algorithm, that has a variance reduction step built in. At a high level, our method updates only once in a ``block'' and within one block we average out noise to reduce the variance. %is a streaming version of the classic power iteration used to compute the SVD in a batch setting. In the rank-$1$ case, our method reduces to a stochastic version of the classic power method.  %Our own flavor is designed to be amenable to theoretical analysis, while performing close to the popular Stochastic Approximation method.

Below, we first illustrate the main ideas of our method as well as our sample complexity proof for the simpler rank-$1$ case. The rank-$1$ and rank-$k$ algorithms are so similar, that we present them in the same panel. We provide the rank-$k$ analysis in Section~\ref{sec:rankk}.  We note that, while our algorithm describes $\{\mathbf{x}_1,\dots,\mathbf{x}_n\}$ as ``input,'' we mean this in the streaming sense: the data are no-where stored, and can never be revisited unless the algorithm explicitly stores them.

\subsection{Rank-One Case}
\label{sec:rank1}
\begin{algorithm}[t]
\caption{Block-Stochastic Power Method \hfill Block-Stochastic Orthogonal Iteration}\label{algo:rank1}
  \begin{algorithmic}[1]
    \INPUT $\{\bx_1, \dots, \bx_n\}$, Block size: $B$ 
    \STATE $\bq_0\sim \mathcal{N}(0, I_{p\times p})$ (Initialization) \hfill $H^i \sim \mathcal{N}(0, I_{p\times p}), 1\leq i\leq k$ (Initialization)\label{stp:init1}
    \STATE $\bq_0 \leftarrow \bq_0/\|\bq_0\|_2$ \hspace{5.22cm} $H \leftarrow Q_0 R_0$ (QR-decomposition)
    \FOR{$\tau=0,\dots, n/B-1$}
    \STATE $\bs_{\tau+1}\leftarrow 0$ \hspace{5.9cm} $S_{\tau+1}\leftarrow 0$
    \FOR{$t= B\tau+1, \dots, B (\tau+1)$}
      \STATE $\bs_{\tau+1} \leftarrow \bs_{\tau+1}+\frac{1}{B}\ip{\bq_\tau}{\bx_t}\bx_t$ \hspace{2.35cm} $S_{\tau+1} \leftarrow S_{\tau+1}+\frac{1}{B}\bx_t \bx_t^{\top}Q_\tau$
    \ENDFOR
    \STATE $\bq_{\tau+1}\leftarrow \bs_{\tau+1}/\|\bs_{\tau+1}\|_2$ \hspace{3.65cm}  $S_{\tau+1}=Q_{\tau+1}R_{\tau+1}$ (QR-decomposition)
    \ENDFOR
    \OUTPUT 
  \end{algorithmic}
\end{algorithm}
We first consider the rank-$1$ case for which each sample $\bx_t$ is generated using: $\bx_t=\bu \bz_t + \bw_t$ where $\bu\in \mathbb{R}^p$ is the principal component that we wish to recover. Our algorithm is a block-wise method where all the $n$ samples are divided in $n/B$ blocks (for simplicity we assume that $n/B$ is an integer). In the $(\tau+1)$-st block, we compute
\begin{equation}\label{eq:sum1}
\bs_{\tau+1}=\left(\frac{1}{B}\sum_{t=B\tau+1}^{B(\tau+1)}\bx_t\bx_t^{\top}\right)\bq_\tau.
\end{equation}   
Then, the iterate $\bq_\tau$ is updated using $\bq_{\tau+1}=\bs_{\tau+1}/\|\bs_{\tau+1}\|_2$. Note that, $\bs_{\tau+1}$ can be easily computed in an online manner where $O(p)$ operations are required per step. Furthermore, storage requirements are also linear in $p$.
\subsubsection{Analysis}
We now present the sample complexity analysis of our proposed method (Algorithm~\ref{algo:rank1}). We show that, using $O(\sigma^4p\log(p)/\epsilon^2)$ samples, Algorithm~\ref{algo:rank1} obtains a solution $\bq_T$ of accuracy $\epsilon$, i.e.\ $\|\bq_T-\bu\|_2\leq \epsilon$.
\begin{theorem}
	Denote the data stream by $\bx_1, \dots, \bx_n$, where $\bx_t\in \mathbb{R}^p, \forall t$ is generated by \eqref{eq:model}. Set the total number of iterations $T=\Omega(\frac{\log(p/\epsilon)}{\log((\sigma^2+.75)/(\sigma^2+.5))})$ and the block size $B=\Omega(\frac{(1+3(\sigma+\sigma^2)\sqrt{p})^2\log(T)}{\epsilon^2})$. Then, with probability $0.99$, $\|\bq_T-\bu\|_2\leq \epsilon$, where $\bq_T$ is the $T$-th iterate of Algorithm~\ref{algo:rank1}. That is, Algorithm~\ref{algo:rank1} obtains an $\epsilon$-accurate solution  with number of samples ($n$) given by: $$n=\tilde{\Omega}\left(\frac{(1+3(\sigma+\sigma^2)\sqrt{p})^2\log(p/\epsilon)}{\epsilon^2\log((\sigma^2+.75)/(\sigma^2+.5))}\right).\vspace*{-5pt}$$
	Note that in the total sample complexity, we use the notation $\tilde{\Omega}(\cdot)$ to suppress the extra $\log(T)$ factor for clarity of exposition, as $T$ already appears in the expression linearly.
\label{thm:rank1}%\vspace*{-10pt}
\end{theorem}
\begin{proof}\vspace*{-3pt}
The proof decomposes the current iterate into the component of the current iterate, $\mathbf{q}_{\tau}$, in the direction of the true principal component (the spike) $\mathbf{u}$, and the perpendicular component, showing that the former eventually dominates. Doing so hinges on three key components: (a) for large enough $B$, the empirical covariance matrix $F_{\tau+1} =\frac{1}{B}\sum_{t=B\tau+1}^{B(\tau+1)}\bx_t\bx_t^{\top}$ is close to the true covariance matrix $M=\bu\bu^{\top}+\sigma^2I$, i.e., $\|F_{\tau+1}-M\|_2$ is small. In the process, we obtain ``tighter'' bounds for $\|\bu^{\top}(F_{\tau+1}-M)\bu\|$ for {\em fixed} $\bu$; (b) with probability $0.99$ (or any other constant probability), the initial point $\bq_0$ has a component of at least $O(1/\sqrt{p})$ magnitude along the true direction $\bu$; (c) after $\tau$ iterations, the error in estimation is at most $O(\gamma^\tau)$ where $\gamma<1$ is a constant. 

There are several results that we use repeatedly, which we collect here, and prove individually in the appendix.

{\bf Lemmas \ref{lem:rank1_conc}, \ref{lem:rank1_conc1} and \ref{lem:rank1_init}}. Let $B$, $T$ and the data stream $\{\mathbf{x}_i\}$ be as defined in the theorem. Then:
\begin{itemize}
\item (Lemma \ref{lem:rank1_conc}):  With probability $1-C/T$, for $C$ a universal constant, we have: $$\left\|\frac{1}{B}\sum_{t}\bx_t\bx_t^{\top}-\bu\bu^{\top}-\sigma^2I\right\|_2\leq \epsilon.$$
\item (Lemma \ref{lem:rank1_conc1}): With probability $1-C/T$, for $C$ a universal constant, we have: $$\bu^{\top}\bs_{\tau+1}\geq \bu^{\top}\bq_\tau (1+\sigma^2)\left(1-\frac{\epsilon}{4(1+\sigma^2)}\right),$$ where $\bs_t=\frac{1}{B}\sum_{B\tau<t\leq B(\tau+1)}\bx_t\bx_t^{\top}\bq_\tau$. 
\item (Lemma \ref{lem:rank1_init}): Let $\bq_0$ be the initial guess for $\bu$, given by Steps 1 and 2 of Algorithm~\ref{algo:rank1}. Then, w.p. $0.99$: $|\ip{\bq_0}{\bu}|\geq \frac{C_0}{\sqrt{p}}$, where $C_0>0$ is a universal constant. 
\end{itemize}

Step (a) is proved in Lemmas~\ref{lem:rank1_conc} and \ref{lem:rank1_conc1}, while Lemma~\ref{lem:rank1_init} provides the required result for the initial vector $\bq_0$. Using these lemmas, we next complete the proof of the theorem. We note that both (a) and (b) follow from well-known results; we provide them for completeness. 

Let $\bq_\tau=\sqrt{1-\delta_\tau}\bu+\sqrt{\delta_\tau}\bg_\tau, 1\leq \tau\leq n/B$, where $\bg_\tau$ is the component of $\bq_\tau$ that is perpendicular to $\bu$ and $\sqrt{1-\delta_\tau}$ is the magnitude of the component of $\bq_\tau$ along $\bu$. Note that $\bg_\tau$ may well change at each iteration; we only wish to show $\delta_{\tau} \rightarrow 0$. 

Now, using Lemma~\ref{lem:rank1_conc1}, % and assuming $B, T$ as given in the theorem, with probability at least $1-C/T$, where $C>0$ is an arbitrary constant, 
the following holds with probability at least $1-C/T$: 
\begin{equation}\bu^{\top}\bs_{\tau+1}\geq \sqrt{1-\delta_\tau} (1+\sigma^2)\left(1-\frac{\epsilon}{4(1+\sigma^2)}\right).\label{eq:rank1_1}\end{equation}
Next, we consider the component of $\bs_{\tau+1}$ that is perpendicular to $\bu$: 
$$  
\bg_{\tau+1}^{\top}\bs_{\tau+1} = \bg_{\tau+1}^{\top}\left(\frac{1}{B}\sum_{t=B\tau+1}^{B(\tau+1)}\bx_t\bx_t^{\top}\right)\bq_{\tau} 
=\bg_{\tau+1}^{\top}(M+E_{\tau})\bq_{\tau},
$$
where $M=\bu\bu^{\top}+\sigma^2I$ and $E_\tau$ is the error matrix: $E_\tau=M-\frac{1}{B}\sum_{t=B\tau+1}^{B(\tau+1)}\bx_t\bx_t^{\top}$. Using Lemma~\ref{lem:rank1_conc}, $\|E_\tau\|_2\leq \epsilon$ (w.p. $\geq 1-C/T$). Hence, w.p. $\geq 1-C/T$: 
\begin{equation}
\bg_{\tau+1}^{\top}\bs_{\tau+1}=\sigma^2\bg_{\tau+1}^{\top}\bq_{\tau}+\|\bg_{\tau+1}\|_2\|E_\tau\|_2\|\bq_\tau\|_2 
\leq \sigma^2\sqrt{\delta_\tau}+\epsilon. \label{eq:rank1_2}
\end{equation}
%\begin{align}
%\bg_{\tau+1}^{\top}\bs_{\tau+1}&=\sigma^2\bg_{\tau+1}^{\top}\bq_{\tau}+\|\bg_{\tau+1}\|_2\|E_\tau\|_2\|\bq_\tau\|_2, \nonumber\\
%&\leq \sigma^2\bg_{\tau+1}^{\top}(\sqrt{1-\delta_\tau}\bu+\sqrt{\delta_\tau}\bg_\tau)+\epsilon, \nonumber\\
%&\leq \sigma^2\sqrt{\delta_\tau}+\epsilon. \label{eq:rank1_2}
%\end{align}

Now, since $\bq_{\tau+1}=\bs_{\tau+1}/\|\bs_{\tau+1}\|_2$, 
\begin{align}
  \delta_{\tau+1}&=(\bg_{\tau+1}^{\top}\bq_{\tau+1})^2=\frac{(\bg_{\tau+1}^{\top}\bs_{\tau+1})^2}{(\bu^{\top}\bs_{\tau+1})^2+(\bg_{\tau+1}^{\top}\bs_{\tau+1})^2}, \nonumber\\
&\stackrel{(i)}{\leq} \frac{(\bg_{\tau+1}^{\top}\bs_{\tau+1})^2}{(1-\delta_\tau)\left(1+\sigma^2-\frac{\epsilon}{4}\right)^2+(\bg_{\tau+1}^{\top}\bs_{\tau+1})^2},\nonumber\\
&\stackrel{(ii)}{\leq} \frac{(\sigma^2\sqrt{\delta_\tau}+\epsilon)^2}{(1-\delta_\tau)\left(1+\sigma^2-\frac{\epsilon}{4}\right)^2+(\sigma^2\sqrt{\delta_\tau}+\epsilon)^2}, \label{eq:rank1_3}
\end{align}
where, $(i)$ follows from \eqref{eq:rank1_1} and $(ii)$ follows from \eqref{eq:rank1_2} along with the fact that $\frac{x}{c+x}$ is an increasing function in $x$ for $c, x \geq 0$. 

Assuming $\sqrt{\delta_{\tau}}\geq 2\epsilon$ and using \eqref{eq:rank1_3} and bounding the failure probability with a union bound, we get (w.p.\ $\geq 1-\tau\cdot C/T$)
\begin{equation}
\label{eq:rank1_4}
    \delta_{\tau+1} \leq \frac{\delta_{\tau}(\sigma^2+1/2)^2}{(1-\delta_\tau)(\sigma^2+3/4)^2+\delta_{\tau}(\sigma^2+1/2)^2} 
    \stackrel{(i)}{\leq}\frac{\gamma^{2\tau}\delta_0}{ 1-(1-\gamma^{2\tau})\delta_0} \stackrel{(ii)}{\leq}C_1\gamma^{2\tau}p,
  \end{equation}
where $\gamma=\frac{\sigma^2+1/2}{\sigma^2+3/4}$ and $C_1>0$ is a global constant. Inequality $(ii)$ follows from Lemma~\ref{lem:rank1_init}; to prove $(i)$, we need one final result: the following lemma shows that the recursion given by \eqref{eq:rank1_4} decreases $\delta_\tau$ at a fast rate. Interestingly, the rate of decrease in error $\delta_\tau$ initially (for small $\tau$) might be sub-linear but for large enough $\tau$ the rate turns out to be linear. We defer the proof to the appendix.
\begin{lemma}\label{lem:rank1_rec}
If for any $\tau \geq 0$ and $0 < \gamma < 1$, we have $\delta_{\tau+1}\leq \frac{\gamma^2\delta_\tau}{1-\delta_\tau+\gamma^2\delta_\tau}$, then, 
$$
\delta_{\tau+1} \leq \frac{\gamma^{2t+2}\delta_0}{1-(1-\gamma^{2t+2})\delta_0}.
$$
\end{lemma}
%\begin{proof}
%The proof follows by induction. The base case (for $\tau=0$) follows trivially.  Now, by the inductive hypothesis, $\delta_\tau\leq \frac{\gamma^{2\tau}\delta_0}{1-(1-\gamma^{2 \tau})\delta_0}.$ That is, $$\frac{1}{\delta_\tau}\geq \frac{1-(1-\gamma^{2\tau})\delta_0}{\gamma^{2\tau}\delta_0}.$$ Finally, by assumption, $$\delta_{\tau+1}\leq \frac{\gamma^2}{\frac{1}{\delta_\tau}-(1-\gamma^2)}\leq  \frac{\gamma^2}{\frac{1-(1-\gamma^{2\tau})\delta_0}{\gamma^{2\tau}\delta_0}-(1-\gamma^2)}.$$
%The lemma follows after simplification of the above given expression. 
%\end{proof}

Hence, using the above equation after $T=O\left(\log(p/\epsilon)/\log{(1 / \gamma)} \right)$ updates, with probability at least $1-C$,  $\sqrt{\delta_{T}}\leq 2\epsilon$. The result now follows by noting that $\|\bu-\bq_{T}\|_2\leq 2\sqrt{\delta_{T}}$. % and by selecting constant $C$ to be small enough, say $\approx 0.001$. 
%
%We use Lemma~\ref{lemma:rank1_err} and \ref{lemma:rank1_init} to prove the above theorem. Lemma~\ref{lemma:rank1_err} shows that assuming initial point $\bq_0$ has at least an $\Omega(1/\sqrt{p})$ component along true principal component $\bu$ and the block size $B$ to be large enough, error $1-(\bu^{\top}\bu)^2$ decreases by at least a constant factor after each update to $\bu$. Lemma~\ref{lemma:rank1_init} shows that the initialization step (Step~\ref{stp:init1} in Algorithm~\ref{algo:rank1}) indeed finds a vector $\bq_0$ (w.p. $0.999$) s.t. $\bq_0^{\top}\bu\geq \frac{C_0}{\sqrt{p}}$ where $C_0>0$ is a universal constant. 
%Now, by Lemma~\ref{lemma:rank1_err} (w.h.p.): $1-(\bq_{\tau}^{\top}\bu)^2\leq \gamma^\tau 1-(\bq_{0}^{\top}\bu)^2$, where $\gamma>0$ is a constant as given by Lemma~\ref{lemma:rank1_err}. Using Lemma~\ref{lemma:rank1_init}, $(\bq_{0}^{\top}\bu)^2$
%Using Lemma~\ref{lemma:rank1_init} 
%  The above theorem follows by using Lemma~\ref{lemma:rank1_err} along with Lemma~\ref{lemma:rank1_init}. 
%\end{proof}
%\begin{proof}
\end{proof}
%\todo{Note that $\delta$ is a bit overloaded.}
{\bf Remark}: Note that in Theorem~\ref{thm:rank1}, the probability of accurate principal component recovery is a constant and does not decay with $p$. One can correct this by either paying a price of $O(\log p)$ in storage, or in sample complexity: for the former, we can run $O(\log p)$ instances of Algorithm~\ref{algo:rank1} in parallel; alternatively, we can run Algorithm~\ref{algo:rank1} $O(\log p)$ times on fresh data each time, using the next block of data to evaluate the old solutions, always keeping the best one. Either approach guarantees a success probability of at least $1 - {1 \over p^{O(1)}}$.

\subsection{General Rank-$k$ Case}
\label{sec:rankk}
In this section, we consider the general rank-$k$ PCA problem where each sample is assumed to be generated using the model of equation \eqref{eq:model}, where $A\in \mathbb{R}^{p\times k}$ represents the $k$ principal components that need to be recovered. Let $A=U\Lambda V^{\top}$ be the SVD of $A$ where $U \in \mathbb{R}^{p\times k}$, $\Lambda, V\in\mathbb{R}^{k\times k}$. The matrices $U$ and $V$ are orthogonal, i.e., $U^{\top}U=I, V^{\top}V=I$, and $\Sigma$ is a diagonal matrix with diagonal elements $\lambda_1\geq \lambda_2\dots \geq\lambda_k$. The goal is to recover the space spanned by $A$, i.e., ${\rm span}(U)$. Without loss of generality, we can assume that $\|A\|_2=\lambda_1=1$. 

Similar to the rank-$1$ problem, our algorithm for the rank-$k$ problem can be viewed as a streaming variant of the classical orthogonal iteration used for SVD. But unlike the rank-$1$ case, we require a more careful analysis as we need to bound spectral norms of various quantities in intermediate steps and simple, crude analysis can lead to significantly worse bounds. Interestingly, the analysis is entirely different from the standard analysis of the orthogonal iteration as there, the empirical estimate of the covariance matrix is fixed while in our case it varies with each block.

For the general rank-$k$ problem, we use the largest-principal-angle-based distance function between any two given subspaces: 
$$
\dist\left({\rm span}(U), {\rm span}(V)\right)=\dist(U, V) =\|U_\perp^{\top}V\|_2=\|V_\perp^{\top}U\|_2,
$$
where $U_\perp$ and $V_\perp$ represent an orthogonal basis of the perpendicular subspace to ${\rm span}(U)$ and ${\rm span}(V)$, respectively. For the spiked covariance model, it is straightforward to see that this is equivalent to the usual PCA figure-of-merit, the expressed variance.

%We now present our main theorem which shows that using not more than $O(C_{\sigma,\lambda_k}p\log(p/\epsilon)/\epsilon^2)$ samples, Algorithm~\ref{algo:rankk} produces a subspace basis $Q_\tau$, which with high probability satisfies $\dist(Q_\tau, U)\leq \epsilon$. $C_{\sigma,\lambda_k}>0$ is a constant depending only on $\sigma, \lambda_k$. 
\begin{theorem}
  \label{thm:rankk}
Consider a data stream, where $\bx_t\in \mathbb{R}^p$ for every $t$ is generated by \eqref{eq:model}, and the SVD of $A\in \mathbb{R}^{p\times k}$ is given by $A=U\Lambda V^{\top}$. Let, wlog, $\lambda_1=1\geq \lambda_2 \geq \dots \geq\lambda_k>0$. Let, 
$$
T=\Omega\left(\log(p/k\epsilon)/\log\left(\frac{\sigma^2+0.75\lambda_k^2}{\sigma^2+0.5\lambda_k^2}\right)\right), \quad
B=\Omega\left(
		\frac{
			\left(
				(1+\sigma)^2\sqrt{k}+\sigma\sqrt{1+\sigma^2}k\sqrt{p}
			\right)^2\log(T)}
			{\lambda_k^4\epsilon^2}
		\right).
$$
%and denote the $t$-th iterate of Algorithm~\ref{algo:rankk} by $Q_t$. Then, with probability at least $0.9$, after 
%iterations of the algorithm, 
Then, after $T$ $B$-size-block-updates, w.p. $0.99$, 
$\dist(U, Q_T)\leq \epsilon$. Hence, the sufficient number of samples for $\epsilon$-accurate recovery of all the top-$k$ principal components is:
\[
	n=\tilde{\Omega}\left(
		{
			\left(
				(1+\sigma)^2\sqrt{k}+\sigma\sqrt{1+\sigma^2}k\sqrt{p}
			\right)^2
			\log(p/k\epsilon) 
				\over 
			\lambda_k^4\epsilon^2
			\log\left(\frac{\sigma^2+0.75\lambda_k^2}{\sigma^2+0.5\lambda_k^2}\right)
		}
	\right).
\]
Again, we use $\tilde{\Omega}(\cdot)$ to suppress the extra $\log(T)$ factor.

%$\|\bu_\tau-\ba\|_2\leq \epsilon$, where $\bu_\tau$ is the $\tau$-th iterate of Algorithm~\ref{algo:rank1} and  $\tau=\Omega(\frac{\log(p/\epsilon)}{\log((\sigma^2+.5)/(\sigma^2+.75))})$. That is, Algorithm~\ref{algo:rank1} obtains $\epsilon$-accurate solution  with number of samples ($n$) given by: $$n=\Omega\left(\frac{(1+\sigma^2\sqrt{p})^2\log(p/\epsilon)}{\epsilon^2\log((\sigma^2+.5)/(\sigma^2+.75))}\right).$$
\end{theorem}
%Please refer to the Remark after Algorithm \ref{algo:rank1} for a couple of ways to get the same result with high probability.

%As in the proof for the rank-$1$ problem, we first show that in each block, $F_{\tau+1}=\frac{1}{B}\sum_{t=B\tau+1}^{B(\tau+1)}\bx_t\bx_t^{\top}$ is close to the true ``covariance'' matrix $M=AA^{\top}+\sigma^2I$, and that the initial iterate $Q_0$ has large enough component along the true subspace, $span(U)$. Finally, we combine these two components to provide a recursion that ensures that $\dist(U, Q_\tau)$ decreases at a fast rate. 
%
%Compared to the rank-$1$ case, we require a more careful analysis as we need to bound spectral norms of various quantities in intermediate steps and simple, crude analysis can lead to significantly worse bounds. Interestingly, the analysis is entirely different from the standard analysis of the orthogonal iteration as there, the empirical estimate of the covariance matrix is fixed while in our case it varies with each block. 
The key part of the proof requires the following additional lemmas that bound the energy of the current iterate along the desired subspace and its perpendicular space (Lemmas \ref{lem:rankk_lb} and {\ref{lem:rankk_ub}), and Lemma \ref{lem:rankk_init}, which controls the quality of the initialization.

{\bf Lemmas \ref{lem:rankk_lb}, \ref{lem:rankk_ub} and \ref{lem:rankk_init}}. Let the data stream, $A$, $B$, and $T$ be as defined in Theorem~\ref{thm:rankk}, $\sigma$ be the variance of noise, $\Ftt=\frac{1}{B}\sum_{B\tau<t\leq B(\tau+1)}\bx_t \bx_t^{\top}$ and $\Qt$ be the $\tau$-th iterate of Algorithm~\ref{algo:rank1}. 
\begin{itemize}
\item ({\bf Lemma \ref{lem:rankk_lb}}): $\forall \ \bv\in \mathbb{R}^{k}$ and $\|\bv\|_2=1$, w.p. $1-5C/T$ we have:
$$\|U^{\top}\Ftt\Qt\bv\|_2\geq (\lambda_k^2+\sigma^2-\frac{\lambda_k^2\epsilon}{4})\sqrt{1-\|U_\perp^{\top}\Qt\|_2^2}.$$

\item ({\bf Lemma \ref{lem:rankk_ub}}): With probability at least $1-4C/T$, %For all $\bv\in \mathbb{R}^{k}$ and $\|\bv\|_2=1$: 
$\|U_\perp^{\top}\Ftt\Qt\|_2\leq \sigma^2\|U_\perp^{\top}\Qt\|_2+\lambda_k^2\epsilon/2$.

\item ({\bf Lemma \ref{lem:rankk_init}}): Let $Q_0 \in \mathbb{R}^{p\times k}$ be sampled uniformly at random as in Algorithm~\ref{algo:rank1}.  Then, w.p. at least $0.99$:  $\sigma_{k}(U^{\top}Q_0)\geq C\sqrt{\frac{1}{kp}}$. 

\end{itemize}

We provide the proof of the lemmas and theorem in the appendix.

\subsection{Perturbation-tolerant Subspace Recovery}\label{sec:robust}
While our results thus far assume $A$ has rank exactly $k$, and $k$ is known {\it a priori}, here we show that both these can be relaxed; hence our results hold in a quite broad setting. %In the previous section, we presented a method for recovering $k$-dimensional subspace $U\in \mathbb{R}^{p\times k}$ assuming that the mixing matrix $A$ is also of {\em exact} rank-$k$. Furthermore, we need to know the rank of $A$ exactly to recover $U$. In this section, we show that these restrictions are not required and our results hold in a general setting as well. 

Let $\bx_t=A\bz_t+\bw_t$ be the $t$-th step sample, with $A=U\Lambda V^T\in \mathbb{R}^{p\times r}$ and $U\in \mathbb{R}^{p\times r}$ where $r\geq k$ is the true rank of $A$ which is unknown. However, we run Algorithm~\ref{algo:rank1} with rank $k$ and the goal is to recover a subspace $Q_T$, s.t., $Q_T$ is contained in $U$. 

We first observe that the largest-principal angle based distance function that we use in the previous section can directly be used for our more general setting. That is, $\dist(U, Q_T)=\|U_\perp^T Q_T\|_2$ measures the component of $Q_T$ ``outside'' the subspace $U$ and the goal is to show that component is $\leq \epsilon$.

Now, our analysis can be easily modified to handle this more general setting as crucially our distance function does not change. Naturally, now the number of samples we require increases according to $r$. In particular, if 
$$	n=\tilde{\Omega}\left(
		{
			\left(
				(1+\sigma)^2\sqrt{r}+\sigma\sqrt{1+\sigma^2}r\sqrt{p}
			\right)^2
			\log(p/r\epsilon) 
				\over 
			\lambda_r^4\epsilon^2
			\log\left(\frac{\sigma^2+0.75\lambda_r^2}{\sigma^2+0.5\lambda_r^2}\right)
		}
	\right),$$
then $\dist(U, Q_T)\leq \epsilon$. Furthermore, if we assume $r\geq C\cdot k$ (or  a large enough constant $C>0$) then the initialization step provides us better distance, i.e., $\dist(U, Q_0)\leq C'/\sqrt{p}$ rather than $\dist(U, Q_0)\leq C'/\sqrt{kp}$ bound if $r=k$. This initialization step enables us to give tighter sample complexity as the $r \sqrt{p}$ in the numerator above can be replaced by $\sqrt{rp}$. % That is, rather than $O(r^2 p)$ samples we need $O(rp)$ samples. 
%%% Local Variables: 
%%% mode: latex
%%% TeX-master: "preprint-cc"
%%% End: 

\section{Experiments}
\label{sec:experiments}

%\begin{figure*}[t!]
%  \begin{tabular}[t!]{cccc}\hspace*{-10pt}
%    \includegraphics[width=.25\textwidth]{results/wbatch.eps}\hspace*{-5pt}&\hspace*{-5pt}
%		\includegraphics[width=.25\textwidth]{results/pt.eps}\hspace*{-5pt}\vspace*{-4pt} &
%%{\bf (a)}&{\bf (b)} \vspace*{-0pt}\\
%    \includegraphics[width=.25\textwidth]{results/bow-nips.eps}\hspace*{-5pt}&\hspace*{-5pt}
%    \includegraphics[width=.25\textwidth]{results/bow.eps}\hspace*{-7pt}\\
%{\bf (c)}& {\bf (c)}& {\bf (c)}&{\bf (d)}\vspace*{-12pt}
%  \end{tabular}
%	\caption{{\bf (a)} Number of samples required for recovery of a single component ($k=1$) from the spiked covariance model, with noise standard deviation $\sigma=0.5$ and desired accuracy $\epsilon=0.05$. 
%%The figure compares batch SVD, and Algorithm \ref{algo:rank1},
%{\bf (b)}\ Fraction of trials in which Algorithm \ref{algo:rank1} successfully recovers the principal component ($k=1$) in the same model, with $\epsilon=0.05$ and $n=1000$ samples,
%{\bf (c)} Explained variance by Algorithm \ref{algo:rank1} compared to the optimal batch SVD, on the NIPS bag-of-words dataset.
%{\bf (d)} Explained variance by Algorithm \ref{algo:rank1} on the NY Times and PubMed datasets.
%}
%\label{fig:exp}
%\end{figure*}

\begin{figure*}[t!] \center
  \begin{tabular}[t!]{cc}\hspace*{-10pt}
    \includegraphics[width=.5\textwidth]{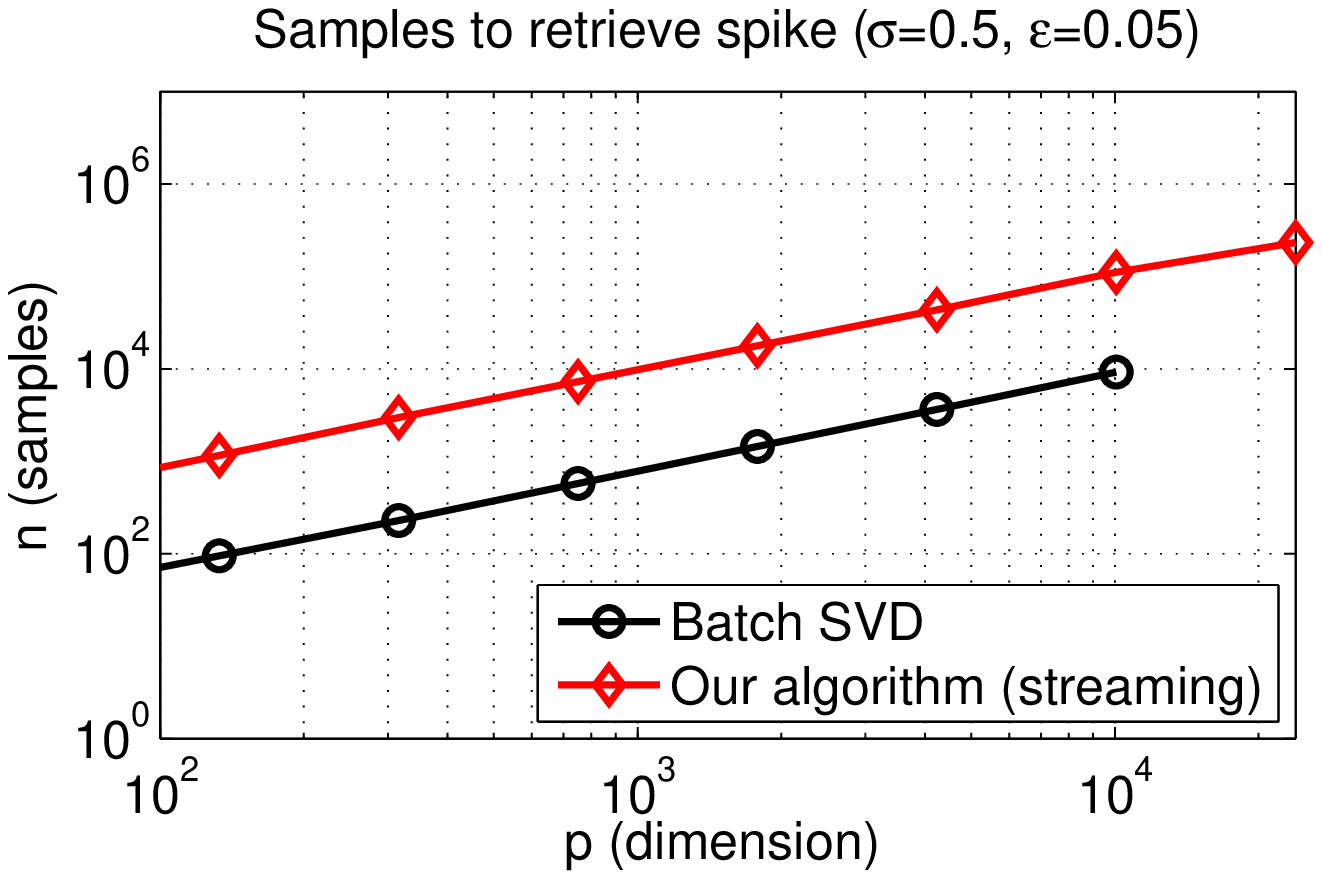}\hspace*{-5pt}&\hspace*{-5pt}
		\includegraphics[width=.5\textwidth]{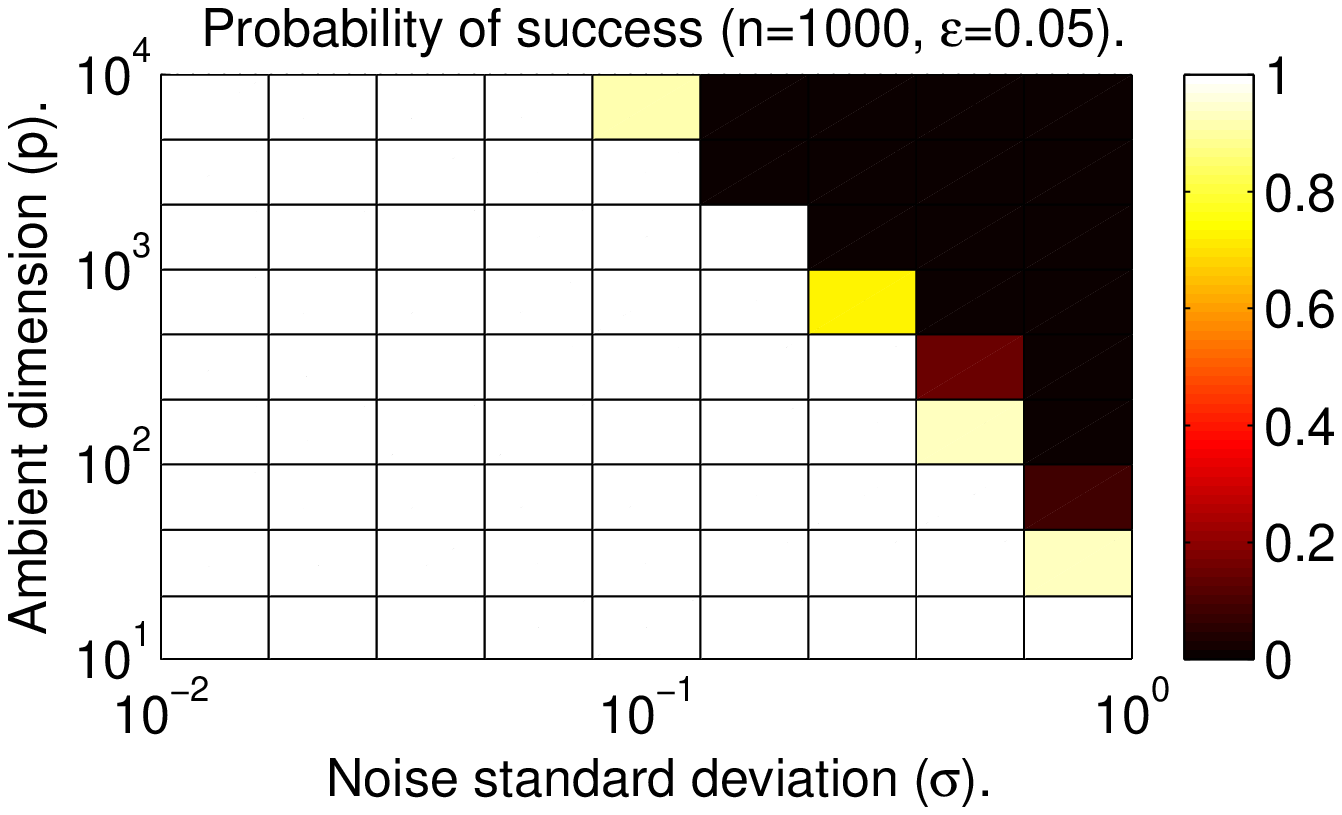}\hspace*{-5pt}\vspace*{-4pt} \\
{\bf {\small (a)}}&{\bf {\small (b)}} \vspace*{-0pt}\\
    \includegraphics[width=.5\textwidth]{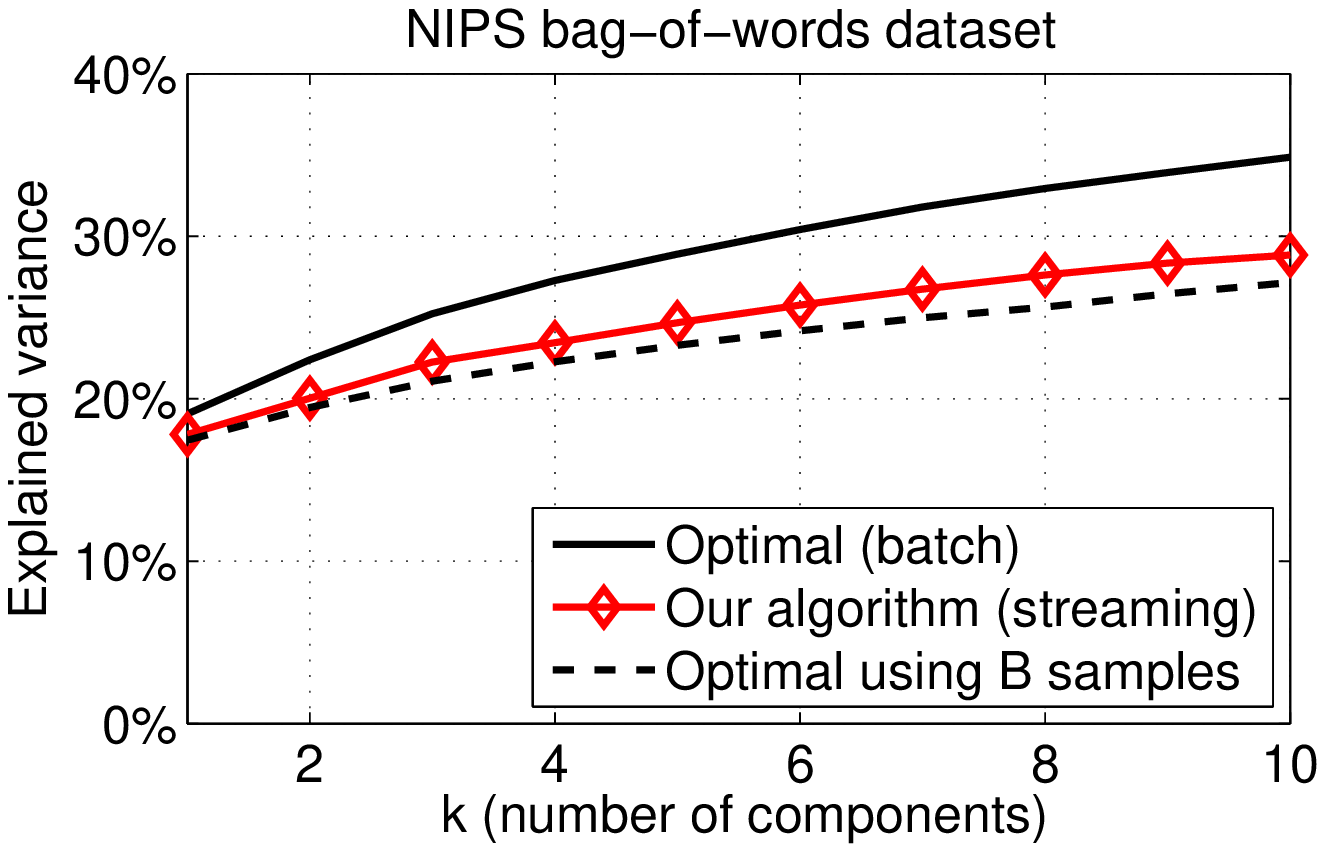}\hspace*{-5pt}&\hspace*{-5pt}
    \includegraphics[width=.5\textwidth]{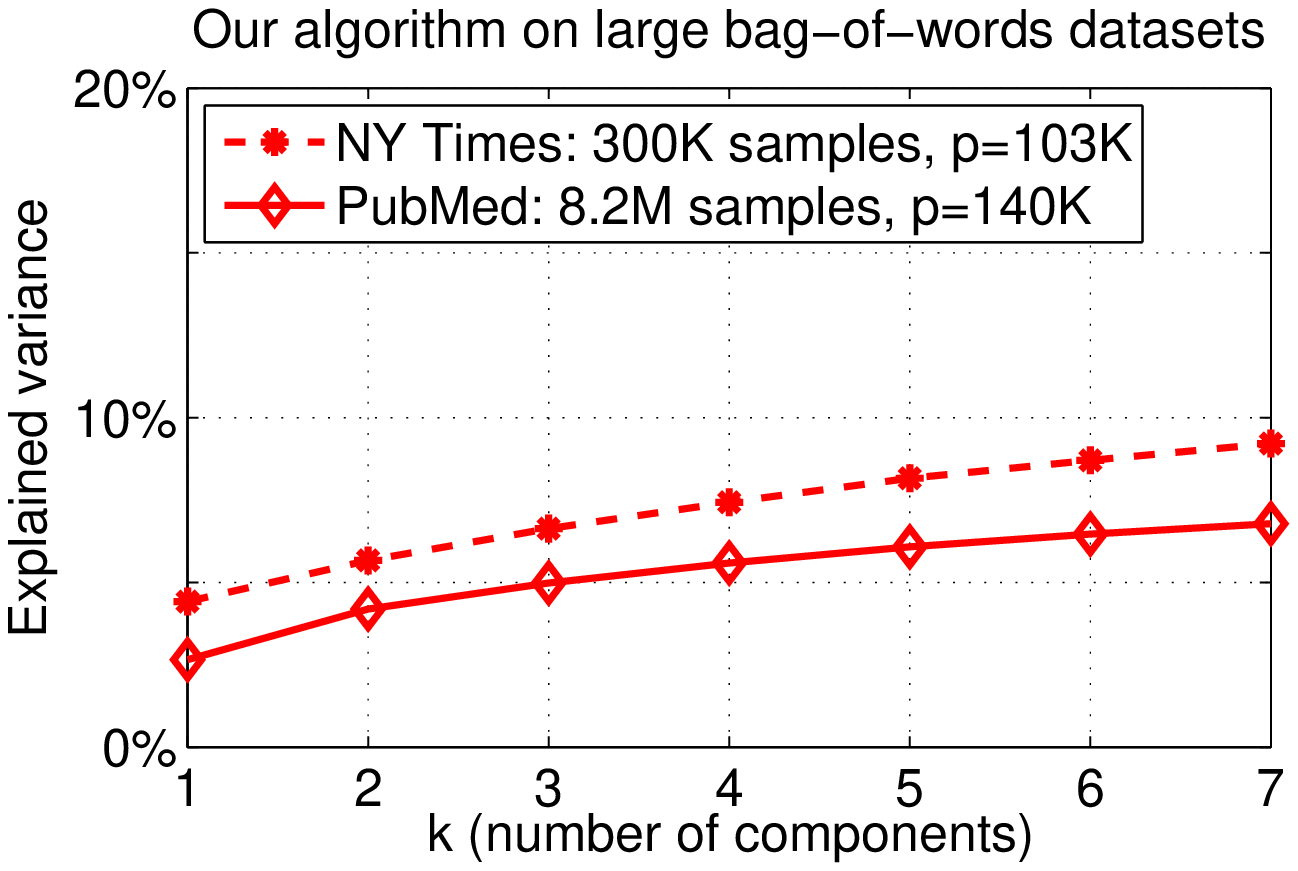}\hspace*{-7pt}\\
{\bf {\small (c)}}&{\bf {\small (d)}}\vspace*{-12pt}
  \end{tabular}
	\caption{{\bf (a)} Number of samples required for recovery of a single component ($k=1$) from the spiked covariance model, with noise standard deviation $\sigma=0.5$ and desired accuracy $\epsilon=0.05$. 
%The figure compares batch SVD, and Algorithm \ref{algo:rank1},
{\bf (b)}\ Fraction of trials in which Algorithm \ref{algo:rank1} successfully recovers the principal component ($k=1$) in the same model, with $\epsilon=0.05$ and $n=1000$ samples,
{\bf (c)} Explained variance by Algorithm \ref{algo:rank1} compared to the optimal batch SVD, on the NIPS bag-of-words dataset.
{\bf (d)} Explained variance by Algorithm \ref{algo:rank1} on the NY Times and PubMed datasets.
}
\label{fig:exp}
\end{figure*}

%\begin{figure}
%\subfloat[]%optionally add here a short text as a label]
%{\label{fig:image_1}
%\includegraphics[width=0.2\textwidth]{results/wbatch}}
%
%\subfloat[]{
%\label{fig:image_2}
%\includegraphics[width=0.2\textwidth]{results/wbatch}}
%
%\subfloat[]{
%\label{fig:image_3}
%\includegraphics[width=0.2\textwidth]{results/wbatch}}
%
%\subfloat[]{
%\label{fig:image_4}
%\includegraphics[width=0.2\textwidth]{results/wbatch}}
%
%\centering
%\caption{\label{fig:multi_panel} some caption text}
%\end{figure}

%It has long been observed empirically that Stochastic Approximation (SA) performs similarly to the Batch-SVD method. Yet no finite-sample analysis of SA is available.
%We now turn to synthetic and real-world numerical experiments. Our synthetic experiments adhere precisely to the model assumptions of our main results in Sections \ref{sec:rank1} and \ref{sec:rankk}, and serve to corroborate these numerically. Our real-world experiments demonstrate that our model works well far beyond the analyzed model. We test our algorithm on real world--and some very big--datasets, using the metric of explained variance.

In this section, we show that, as predicted by our theoretical results, our algorithm performs close to the optimal batch SVD. We provide the results from simulating the spiked covariance model, and demonstrate the phase-transition in the probability of successful recovery that is inherent to the statistical problem. Then we stray from the analyzed model and performance metric and test our algorithm on real world--and some very big--datasets, using the metric of explained variance.

In the experiments for Figures~\ref{fig:exp} (a)-(b), 
we draw data from the generative model of \eqref{eq:model}. Our results are averaged over at least $200$ independent runs.
Algorithm \ref{algo:rank1} uses
the block size prescribed in Theorem~\ref{thm:rankk}, with the empirically tuned constant of $0.2$. As expected, our algorithm exhibits linear scaling with respect to the ambient dimension $p$ -- the same as the batch SVD. The missing point on batch SVD's curve (Figure~\ref{fig:exp}(a)), corresponds to $p>2.4\cdot10^4$. Performing SVD on a dense $p \times p$ matrix, either fails or takes a very long time on most modern desktop computers; in contrast, our streaming algorithm easily runs on this size problem. The phase transition plot in Figure~\ref{fig:exp}(b) shows the empirical sample complexity on a large class of problems and corroborates the scaling with respect to the noise variance we obtain theoretically.

Figures~\ref{fig:exp} (c)-(d) complement our complete treatment of the spiked covariance model, with some out-of-model experiments. We used three bag-of-words datasets from \cite{porteous_fast_2008}. We evaluated our algorithm's performance with respect to the fraction of explained variance metric: given the $p \times k$ matrix $V$ output from the algorithm, and all the provided samples in matrix $X$, the fraction of explained variance is defined as 
$\tr(V^TXX^TV)/\tr(XX^T)$.
 To be consistent with our theory, for a dataset of $n$ samples of dimension $p$, we set the number of blocks to be $T=\lceil \log(p) \rceil $ and the size of blocks to $B=\lfloor n / T \rfloor$ in our algorithm.
 The NIPS dataset is the smallest, with $1500$ documents and $12$K words and allowed us to compare our algorithm with the optimal, batch SVD. We had the two algorithms work on the document space ($p=1500$) and report the results in Figure~\ref{fig:exp}(c). The dashed line represents the optimal using $B$ samples. The figure is consistent with our theoretical result: our algorithm performs as well as the batch, with an added $\log(p)$ factor in the sample complexity.

 Finally, in Figure~\ref{fig:exp} (d), we show our algorithm's ability to tackle very large problems. Both the NY Times and PubMed datasets are of prohibitive size for traditional batch methods -- the latter including $8.2$ million documents on a vocabulary of $141$ thousand words -- so we just report the performance of Algorithm~\ref{algo:rank1}. It was able to extract the top $7$ components for each dataset in a few hours on a desktop computer.
A second pass was made on the data to evaluate the results, and we saw 7-10 percent of the variance explained on spaces with $p>10^4$.

\clearpage
\bibliography{preprint}
\bibliographystyle{plain}
\clearpage
\appendix
\setcounter{page}{1}
%{\center{{\Large {\bf Memory Limited, Streaming PCA: Supplemental Material}}}}

%\vspace{0.5cm}

\section{Lemmas from Section \ref{sec:rank1}}
We first give the statement of all the Lemmas whose proofs we omitted in the body of the paper. Then we provide some results from the literature -- what we call Preliminaries -- and then we prove Theorem \ref{thm:rankk} and the supporting lemmas.

\begin{lemma}\label{lem:rank1_conc}
Let $B$, $T$ and the data stream $\{\mathbf{x}_t\}$ be as defined in Theorem~\ref{thm:rank1}. Then, w.p. $1-C/T$ we have: $$\left\|\frac{1}{B}\sum_{t}\bx_t\bx_t^{\top}-\bu\bu^{\top}-\sigma^2I\right\|_2\leq \epsilon.$$
\end{lemma}

\begin{lemma}\label{lem:rank1_conc1}
Let $B$, $T$ and the data stream $\{\mathbf{x}_t\}$ be as defined in Theorem~\ref{thm:rank1}. Then, w.p. $1-C/T$ we have: $$\bu^{\top}\bs_{\tau+1}\geq \bu^{\top}\bq_\tau (1+\sigma^2)\left(1-\frac{\epsilon}{4(1+\sigma^2)}\right),$$ where $\bs_t=\frac{1}{B}\sum_{B\tau<t\leq B(\tau+1)}\bx_t\bx_t^{\top}\bq_\tau$. 
\end{lemma}

\begin{lemma}\label{lem:rank1_init}
  Let $\bq_0$ be the initial guess for $\bu$, given by Steps 1 and 2 of Algorithm~\ref{algo:rank1}. Then, w.p. $0.99$: $|\ip{\bq_0}{\bu}|\geq \frac{C_0}{\sqrt{p}}$, where $C_0>0$ is a universal constant. 
\end{lemma}

\begin{lemma}%\label{lem:rank1_rec}
If for any $\tau \geq 0$ and $0 < \gamma < 1$, we have $\delta_{\tau+1}\leq \frac{\gamma^2\delta_\tau}{1-\delta_\tau+\gamma^2\delta_\tau}$, then, 
$$
\delta_{\tau+1} \leq \frac{\gamma^{2t+2}\delta_0}{1-(1-\gamma^{2t+2})\delta_0}.
$$
\end{lemma}

\section{Lemmas from Section \ref{sec:rankk}}

\begin{lemma}
  \label{lem:rankk_lb}
Let $\mathcal{X}$, $A$, $B$, and $T$ be as defined in Theorem~\ref{thm:rankk}. Also, let $\sigma$ be the variance of noise, $\Ftt=\frac{1}{B}\sum_{B\tau<t\leq B(\tau+1)}\bx_t \bx_t^{\top}$ and $\Qt$ be the $\tau$-th iterate of Algorithm~\ref{algo:rank1}. 
Then, $\forall \ \bv\in \mathbb{R}^{k}$ and $\|\bv\|_2=1$, w.p. $1-5C/T$ we have:
$$\|U^{\top}\Ftt\Qt\bv\|_2\geq (\lambda_k^2+\sigma^2-\frac{\lambda_k^2\epsilon}{4})\sqrt{1-\|U_\perp^{\top}\Qt\|_2^2}.$$ 
\end{lemma}

\begin{lemma}
  \label{lem:rankk_ub}
Let $\mathcal{X}$, $A$, $B$, $\Ftt$, $\Qt$ be as defined in Lemma~\ref{lem:rankk_lb}. Then, w.p. $1-4C/T$, %For all $\bv\in \mathbb{R}^{k}$ and $\|\bv\|_2=1$: 
$\|U_\perp^{\top}\Ftt\Qt\|_2\leq \sigma^2\|U_\perp^{\top}\Qt\|_2+\lambda_k^2\epsilon/2$.
\end{lemma}

\begin{lemma}
  \label{lem:rankk_init}
Let $Q_0\in \mathbb{R}^{p\times k}$ be sampled uniformly at random from the set of all $k$-dimensional subspaces (see Initialization Steps of Algorithm~\ref{algo:rank1}).  
%space of $k$-dimensional orthogonal basis generated by selecting using the QR-decomposition of $H=[\bx_1 \bx_2 \dots \bx_k]$ , where $\bx_1, \dots, \bx_k$ are the first $k$ observed samples. 
Then, w.p. at least $0.99$:  $\sigma_{k}(U^{\top}Q_0)\geq C\sqrt{\frac{1}{kp}}$, where $C>0$ is a global constant. 
\end{lemma}

\section{Preliminaries}
\label{app:prelim}
\begin{lemma}[Lemma 5.4 of \cite{vershynin_introduction_2010}]\label{lem:net_vershynin}
  Let A be a symmetric $k \times k$ matrix, and let $\N_\epsilon$ be an $\epsilon$-net of $S^{k-1}$ for some $\epsilon \in [0, 1)$. Then,
	$$\|A\|_2\leq {1 \over (1-2\epsilon)} \sup_{\bx\in \N_\epsilon}|\langle A\bx,\bx\rangle|.$$
\end{lemma}
\begin{lemma}[Proposition 2.1 of \cite{vershynin_how_2010}]\label{lem:cov_vershynin}
  Consider independent random vectors $\bx_1, \dots, \bx_n$ in $\mathbb{R}^{p}$, $n\geq p$, which have sub-Gaussian distribution with parameter $1$. Then for every $\delta > 0$ with probability at least $1-\delta$ one has,
$$\|\frac{1}{n}\sum_{i=1}^n\bx_i\bx_i^T-\mathbb{E}[\bx_i\bx_i^T]\|_2\leq C \sqrt{\log (2/\delta)}\sqrt{\frac{p}{n}}. $$
\end{lemma}
\begin{lemma}[Corollary 3.5 of \cite{vershynin_introduction_2010}]\label{lem:smax_vershynin}
  Let $A$ be an $N \times n$ matrix whose entries are independent standard normal random variables. Then for every $t\geq 0$, with probability at least $1-2\exp(-t^2/2)$ one has,
$$\sqrt{N}-\sqrt{n}-t\leq \sigma_k(A)\leq \sigma_1(A)\leq \sqrt{N}+\sqrt{n}+t.$$
\end{lemma}
\begin{lemma}[Theorem 1.2 of \cite{rudelson_smallest_2009}]\label{lem:smin_vershynin}
  Let $\zeta_1,\dots,\zeta_n$ be independent centered real random variables with variances at least $1$ and subgaussian moments bounded by $B$. Let $A$ be an $k\times k$ matrix whose rows are independent copies of the random vector $(\zeta_1,\dots,\zeta_n)$. Then for every $\epsilon\geq 0$ one has
$$\Pr(\sigma_{min}(A)\leq \epsilon/\sqrt{k})\leq C\epsilon+c^n,$$
where $C>0$ and $c\in (0,1)$ depend only on $B$. Note that $B=1$ for the standard Gaussian variables. 
\end{lemma}
\begin{lemma}
  \label{lem:rankk_op}
  Let $\bx_i\in \mathbb{R}^{m}, 1\leq i\leq B$ be i.i.d. standard multivariate normal variables. Also, $\by_i\in \mathbb{R}^{n}$ are also  i.i.d. normal variables and are independent of $\bx_i, \forall i$. Then, w.p. $1-\delta$,
$$\left\|\frac{1}{B}\sum_i \bx_i \by_i^{\top}\right\|_2\leq  \sqrt{\frac{C\max(m,n)\log(2/\delta)}{B}}.$$
\end{lemma}
\begin{proof}
 Let $M=\sum_i \bx_i \by_i^T$ and let $m>n$. Then, the goal is to show that, the following holds w.p. $1-\delta$: $\frac{1}{B}\|M\bv\|_2\leq \sqrt{\frac{Cm\log(2/\delta)}{B}}$ for all $\bv\in \mathbb{R}^{n}$ s.t. $\|\bv\|_2=1$. 

We prove the lemma by first showing that the above mentioned result holds for any {\em fixed} vector $v$ and then use standard epsilon-net argument to prove it for all $\bv$. 

Let $\N$ be the $1/4$-net of $S^{n-1}$. Then, using Lemma 5.4 of \cite{vershynin_introduction_2010} (see Lemma~\ref{lem:net_vershynin}),
\begin{equation}
  \label{eq:op1}
  \|\frac{1}{Bm}M^TM\|_2\leq 2 \max_{\bv\in \N}\frac{1}{Bm}\|M\bv\|_2^2. 
\end{equation}
Now, for any fixed $\bv$: $M\bv=\sum_i \bx_i\by_i^T\bv=\sum_i \bx_i c_i$, where $c_i=\by_i^T\bv\sim N(0,1)$. Hence, 
$$\|M\bv\|_2^2=\sum_{\ell=1}^m(\sum_{i=1}^Bx_{i\ell}c_i)^2.$$
Now, $\sum_{i=1}^Bx_{i\ell}c_i\sim N(0, \|c\|_2^2)$ where $c^T=\left[c_1\ c_2\cdots c_B\right]$. 
Hence, $\sum_{i=1}^Bx_{i\ell}c_i=\|c\|_2h_\ell$ where $h_\ell\sim N(0,1)$. 

Therefore, $\|M\bv\|_2^2=\|c\|_2^2\|h\|_2^2$ where $h^T=[h_1\ h_2\cdots h_B]$. Now, 
\begin{multline}
  Pr(\frac{\|c\|_2^2\|h\|_2^2}{Bm}\geq 1+\gamma)\leq Pr(\frac{\|c\|_2^2}{B}\geq \sqrt{1+\gamma})+Pr(\frac{\|h\|_2^2}{m}\geq \sqrt{1+\gamma})\\
\stackrel{\zeta_1}{\leq} 2\exp(-\frac{B\gamma^2}{32})+2\exp(-\frac{m\gamma^2}{32})\leq 4\exp(-\frac{m\gamma^2}{32}),\label{eq:op2}
\end{multline}
where $0<\gamma<3$ and $\zeta_1$ follows from Lemma~\ref{lem:smax_vershynin}. 

Using \eqref{eq:op1}, \eqref{eq:op2}, the following holds with probability $(1-9^{n+1}e^{-\frac{m\gamma^2}{32}})$: 
\begin{equation}
  \label{eq:op3}
  \frac{\|M\|_2^2}{Bm}\leq 1+2\gamma. 
\end{equation}
The result now follows by setting $\gamma$ appropriately and assuming $n<C m$ for small enough $C$. 
\end{proof}

\section{Proof of Theorem \ref{thm:rankk}}

Recall that our algorithm proceeds in a blockwise manner; for each block of samples, we compute 
\begin{equation}
  \label{eq:sumk}
S_{\tau+1}=\left(\frac{1}{B}\sum_{t=B\tau+1}^{B(\tau+1)}\bx_t\bx_t^{\top}\right)Q_\tau,
\end{equation}
where $Q_\tau\in \mathbb{R}^{p\times k}$ is the $\tau$-th block iterate and is an orthogonal matrix, i.e., $Q_\tau^{\top}Q_\tau=I_{k\times k}$. Given $S_{\tau+1}$, the next iterate, $Q_{\tau+1}$, is computed by the QR-decomposition of $S_{\tau+1}$. That is, 
\begin{equation}\label{eq:updatek}S_{\tau+1}=Q_{\tau+1}R_{\tau+1},\end{equation}
where $R_{\tau+1}\in \mathbb{R}^{k\times k}$ is an upper-triangular matrix.

\begin{proof}
By using update for $Q_{\tau+1}$ (see \eqref{eq:sumk}, \eqref{eq:updatek}): 
\begin{equation}
  \label{eq:updatek1}
  Q_{\tau+1}R_{\tau+1}=F_{\tau+1}Q_{\tau},
\end{equation}
where $\Ftt=\frac{1}{B}\sum_{B\tau<t\leq B(\tau+1)}\bx_t \bx_t^{\top}$. 
That is, 
\begin{equation}
  \label{eq:k1}
  U_\perp^{\top}Q_{\tau+1}R_{\tau+1}\bv=U_\perp^{\top}F_{\tau+1}Q_{\tau}\bv,\quad \forall \bv\in \mathbb{R}^k,
\end{equation}
where $U_\perp$ is an orthogonal basis of the subspace orthogonal to ${\rm span}(U)$. 
Now, let $\bv_1$ be the singular vector corresponding to the largest singular value, then: 
\begin{align}
&\|U_\perp^{\top}Q_{\tau+1}\|_2^2=\frac{\|U_\perp^{\top}Q_{\tau+1}\bv_1\|_2^2}{\|\bv_1\|_2^2}=\frac{\|U_\perp^{\top}Q_{\tau+1}\Rtt\tv_1\|_2^2}{\|\Rtt\tv_1\|_2^2}\nonumber\\
&\stackrel{(i)}{=}\frac{\|U_\perp^{\top}Q_{\tau+1}\Rtt\tv_1\|_2^2}{\|U^{\top}\Qtt\Rtt\tv_1\|_2^2+\|U_\perp^{\top}\Qtt\Rtt\tv_1\|_2^2}\nonumber\\
&\stackrel{(ii)}{=}\frac{\|U_\perp^{\top}\Ftt\Qt\tv_1\|_2^2}{\|U^{\top}\Ftt\Qt\tv_1\|_2^2+\|U_\perp^{\top}\Ftt\Qt\tv_1\|_2^2}. \label{eq:k2}
\end{align}
where $\tv_1=\frac{\Rtt^{-1}\bv_1}{\|\Rtt^{-1}\bv_1\|_2}$. $(i)$ follows as $\Qtt$ is an orthogonal matrix and $[U\ U_\perp]$ form a complete orthogonal basis; $(ii)$ follows by using \eqref{eq:updatek1}. The existence of $\Rtt^{-1}$ follows using Lemma~\ref{lem:rankk_lb} along with the fact that $\sigma_k(\Rtt)=\|\Rtt\zeta_0\|_2\geq \|U^{\top}\Qtt\Rtt\zeta_0\|_2=\|U^{\top}\Ftt\Qt\zeta_0\|_2>0$, where $\zeta_0$ is the singular vector of $\Rtt$ corresponding to its smallest singular value, $\sigma_k(\Rtt)$. 

Now, using \eqref{eq:k2} with Lemmas~\ref{lem:rankk_lb}, \ref{lem:rankk_ub} and using the fact that $x/(x+c)$ is an increasing function of $x$, for all $x>0$, we get (w.p. $\geq 1-2C/T$):{\small
\begin{equation*}
  \|U_\perp^{\top}Q_{\tau+1}\|_2^2\leq \frac{(\sigma^2 \|U_\perp^{\top}Q_{\tau}\|_2+\lambda_k^2\epsilon/2)^2}{(\lambda_k^2+\sigma^2-\frac{\lambda_k^2\epsilon}{4})^2(1-\|U_\perp^{\top}Q_{\tau}\|_2^2)+(\sigma^2 \|U_\perp^{\top}Q_{\tau}\|_2+0.5\lambda_k^2\epsilon)^2}. 
\end{equation*}}
Now, assuming $\epsilon\leq \|U_\perp^{\top}Q_{\tau}\|_2^2$, using the above equation and by using union bound, we get (w.p. $\geq 1-2\tau C /T$): 
\begin{equation}
  \label{eq:k3}
   \|U_\perp^{\top}Q_{\tau+1}\|_2^2\leq \frac{\gamma^2\|U_\perp^{\top}Q_{\tau}\|_2^2}{1-\|U_\perp^{\top}Q_{\tau}\|_2^2+\gamma^2\|U_\perp^{\top}Q_{\tau}\|_2^2},
\end{equation}
where $\gamma=\frac{\sigma^2+\lambda_k^2/2}{\sigma^2+3\lambda_k^2/4}<1$ for $\lambda_k>0$. Using Lemma~\ref{lem:rank1_rec} along with the above equation, we get (w.p. $\geq 1-2\tau C /T$): 
$$\|U_\perp^{\top}Q_{\tau+1}\|_2^2\leq \gamma^{2\tau}\frac{\|U_\perp^{\top}Q_{0}\|_2^2}{1-\|U_\perp^{\top}Q_{0}\|_2^2}.$$
Now, using Lemma~\ref{lem:rankk_init} we know that $\|U_\perp^{\top}Q_{0}\|_2^2$ is at most $1-\Omega(1/(kp))$.  Hence, for $T=O(\log(p/\epsilon)/\log(1/\gamma)$, we get: $\|U_\perp^{\top}Q_{T}\|_2^2\leq \epsilon$. Furthermore, we require $B$ (as mentioned in the Theorem) %=\Omega\left(\frac{(1+\sigma^2)(k+\sigma^2p)}{\lambda_k^4\epsilon^2}\right)$ 
samples per block. Hence, the total sample complexity bound is given by $O(BT)$, concluding the proof.%This concludes the proof.% by setting $\epsilon'=\epsilon/ \lambda_k^2$. 
\end{proof}

\section{Proof of Lemma~\ref{lem:rank1_conc} }
\label{app:rank1_conc}
\begin{proof}
  Note that, 
  \begin{multline}
    \frac{1}{B}\sum_{t}\bx_t\bx_t^{\top}-\bu\bu^{\top}-\sigma^2I=\bu\bu^{\top}\frac{1}{B}\sum_{t}(z_t^2-1)+\\\frac{1}{B}\sum_t(\bw_t\bw_t^{\top}-\sigma^2I)+\frac{1}{B}\sum_tz_t\bw_t\bu^{\top}+\frac{1}{B}\bu\sum_tz_t\bw_t^{\top}. \label{eq:rc1}
  \end{multline}
We now individually bound each of the above given terms in the RHS. Using standard tail bounds for covariance estimation (see Lemma~\ref{lem:cov_vershynin}), we can bound the first two terms (w.p. $1-2C/T$): 
\begin{align}
  &\frac{1}{B}\left|\sum_{t}(z_t^2-1)\right|\leq \sqrt{\frac{C\log(T)}{B}},\nonumber\\
&\|\frac{1}{B}\sum_t(\bw_t\bw_t^{\top}-\sigma^2I)\|_2\leq\sigma^2 \sqrt{\frac{C_1p\log(T)}{B}}. \label{eq:rc2}
\end{align}
Similarly, using Lemma~\ref{lem:rankk_op}, we can bound the last two terms in \eqref{eq:rc1} (w.p. $1-2C/T$): 
\begin{equation}
  \label{eq:rc3}
  \|\frac{1}{B}\sum_tz_t\bw_t\bu^{\top}\|_2=\|\frac{1}{B}\bu\sum_tz_t\bw_t^{\top}\|_2\leq \sigma\sqrt{\frac{C_1p\log(T)}{B}}. 
\end{equation}
The lemma now follows by  using \eqref{eq:rc1}, \eqref{eq:rc2}, \eqref{eq:rc3} along with  $B$ as given by Theorem~\ref{thm:rank1}. 
\end{proof}

\section{Proof of Lemma~\ref{lem:rank1_conc1} }
\label{app:rank1_conc1}
\begin{proof}
Let $\bq_\tau=\sqrt{1-\delta_\tau}\bu+\sqrt{\delta_\tau}\bu_\tau^\perp$, where $\bu_\tau^\perp$ is the component of $\bq_\tau$ that is orthogonal to $\bu$. Now, 
{\small\begin{align}
  &\bu^{\top}\bs_{\tau+1}=\frac{1}{B}\sum_t(\bu^{\top}\bx_t)(\bx_t^{\top}\bq_t)\nonumber\\&=\frac{1}{B}\sum_t(z_t+\bu^{\top}\bw_t)(\sqrt{1-\delta_\tau}(z_t+\bu^{\top}\bw_t)+\sqrt{\delta_\tau}\bw_t^{\top}\bu_\tau^\perp)\nonumber\\
&=\frac{\sqrt{1-\delta_\tau}}{B}\sum_t(z_t+\bu^{\top}\bw_t)^2+\frac{\sqrt{\delta_\tau}}{B}\sum_t(z_t+\bu^{\top}\bw_t)\bw_t^{\top}\bu_\tau^\perp. \label{eq:r1c1}
\end{align}}
Now, the first term above is a summation of $B$ i.i.d. chi-square variables and hence using standard results (see Lemma~\ref{lem:smax_vershynin}),  w.p. $(1-C/T)$: 
\begin{equation}
  \label{eq:r1c2}
\frac{1}{B}  \sum_t(z_t+\bu^{\top}\bw_t)^2\geq (1+\sigma^2)(1-\sqrt{\frac{C\log(2T)}{B}}). 
\end{equation}
Also, $\bw_t^{\top}\bu$ and $\bw_t^{\top}\bu^\perp_\tau$ are independent random variables, as both $\bw_t^{\top}\bu$, $\bw_t^{\top}\bu^\perp_\tau$ are Gaussians and $E[\bw_t^{\top}\bu^\perp_\tau\bu^{\top}\bw_t]=0$. Hence, using Lemma~\ref{lem:rankk_op}, the following holds with probability $\geq 1-4C/T$:  
\begin{multline}
  \label{eq:r1c3}
  \|\frac{1}{B}\sum_t(z_t+\bu^{\top}\bw_t)\bw_t^{\top}\bu_\tau^\perp\|_2\leq \sigma\sqrt{1+\sigma^2}\sqrt{\frac{C\log(T)}{B}}\stackrel{(i)}{\leq} \sigma\sqrt{1+\sigma^2}\sqrt{\frac{C_1p\log(T)}{B(1-\delta_0)}}\sqrt{1-\delta_\tau}, 
\end{multline}
where $(i)$ follows by using inductive hypothesis (i.e., $\sqrt{1-\delta_\tau}> \sqrt{1-\delta_{\tau-1}}$, induction step follows as we show that the error decreases at each step) and Lemma~\ref{lem:rank1_init}. 

The lemma now follows by using \eqref{eq:r1c1}, \eqref{eq:r1c2}, \eqref{eq:r1c3} and by setting $B, T$ appropriately. 
\end{proof}

\section{Proof of Lemma \ref{lem:rank1_init}}

\begin{proof}
Using standard tail bounds for Gaussians (see Lemma~\ref{lem:smax_vershynin}), $\|\bq_0\|_2\leq 2\sqrt{p}$ with probability $1-\exp(-C_1p)$, where $C_1>0$ is a universal constant. Furthermore, $(\|\bq_0\|_2\bq_0)^{\top}\bu \sim N(0,1)$. Hence, there exists $C_0>0$, s.t., with probability $0.99$, $|(\|\bq_0\|_2\bq_0)^T\bu|\geq C_0$. Hence, $|\bq_0^{\top}\bu|\geq \frac{C_0}{2\sqrt{p}}$.
\end{proof}

\section{Proof of Lemma \ref{lem:rank1_rec}}
\begin{proof}
We prove the lemma using induction. The base case (for $\tau=0$) follows trivially. 

Now, by the inductive hypothesis, $\delta_\tau\leq \frac{\gamma^{2t}\delta_0}{1-(1-\gamma^{2t})\delta_0}.$ That is, $$\frac{1}{\delta_\tau}\geq \frac{1-(1-\gamma^{2t})\delta_0}{\gamma^{2t}\delta_0}.$$ Finally, by assumption, $$\delta_{\tau+1}\leq \frac{\gamma^2}{\frac{1}{\delta_\tau}-(1-\gamma^2)}\leq  \frac{\gamma^2}{\frac{1-(1-\gamma^{2t})\delta_0}{\gamma^{2t}\delta_0}-(1-\gamma^2)}.$$
The lemma follows after simplification of the above given expression. 
%Note that, as $\gamma<1$, RHS of the above equation is an increasing function of $ Hence, using the above two equations: 
%$$$$
\end{proof}

\section{Proof of Lemma~\ref{lem:rankk_lb} }
\label{app:rankk_lb}
\begin{proof}
Using the generative model \eqref{eq:model}, we get: 
\begin{multline}
  U^{\top}\Ftt\Qt\bv=\Lambda\left(\frac{1}{B}\sum_t \bz_t\bz_t^{\top}\right)\Lambda U^{\top}\Qt\bv+\left(\frac{1}{B}\sum_t U^{\top}\bw_t\bw_t^{\top}U\right)U^{\top}\Qt\bv\\+(\frac{1}{B}\sum_t U^{\top} \bw_t\bz_t^{\top})\Lambda U^{\top}\Qt\bv+\Lambda(\frac{1}{B}\sum_t \bz_t\bw_t^{\top}U)U^{\top}\Qt\bv+\left(\frac{1}{B}\sum_t(\Lambda\bz_t+U^{\top}\bw_t)\bw_t^{\top}U_\perp U_\perp^{\top}\Qt\right)\bv. 
\label{eq:k_lb1}
\end{multline}
Note that in the equation and rest of the proof, $t$ varies from $B\tau< t\leq B(\tau+1)$. 
%Consider $U^{\top}\Ftt\Qt\bv=  $

We now show that each of the five terms in the above given equation concentrate around their respective means. Also, let $\by_t=U^{\top}\bw_t$ and $\by^\perp_t=U_\perp^{\top}\bw_t$. Note that, $\by_t\sim N(0, \sigma^2 I_{k\times k})$ and $\by^\perp_t\sim N(0, \sigma^2I_{(p-k)\times (p-k)})$. \\
\noindent {\bf (a)}: Consider the first term in \eqref{eq:k_lb1}. Using $\|A\bv\|_2\leq \|A\|_2\|\bv\|_2$ and the assumption that $\lambda_1=1$, we get: $\|\Lambda\left(\frac{1}{B}\sum_t \bz_t\bz_t^{\top}-I\right)\Lambda U^{\top}\Qt\bv\|_2\leq \|\left(\frac{1}{B}\sum_t \bz_t\bz_t^{\top}-I\right) \|_2\|U^{\top}\Qt\bv\|_2$. 
%$\leq \|\frac{1}{B}\sum_t \bz_t\bz_t^{\top}-I\|_2$, where the last inequality follows from the assumption that $\lambda_1=1$. 
Using Lemma~\ref{lem:cov_vershynin} we get (w.p. $1-C/T$): 
$$\|\frac{1}{B}\sum_t \bz_t\bz_t^{\top}-I\|_2\leq \sqrt{\frac{C_1k\log(T)}{B}}.$$
That is,
\begin{equation}
  \label{eq:k_lb2}\hspace*{-10pt}
  \|\Lambda(\frac{1}{B}\sum_t \bz_t\bz_t^{\top}-I)\Lambda U^{\top}\Qt\bv\|_2\leq \sqrt{\frac{C_1k\log(T)}{B}}\|U^{\top}\Qt\bv\|_2.
\end{equation}
\noindent{\bf (b)}: Similarly, the second term in \eqref{eq:k_lb1} can be bounded as (w.p. $1-C/T$):
\begin{equation}\label{eq:k_lb3}
\|\left(\frac{1}{B}\sum_t U^{\top}\bw_t\bw_t^{\top}U-\sigma^2I\right)U^{\top}\Qt\bv\|_2\leq \sigma^2\sqrt{\frac{C_1k\log(T)}{B}}\|U^{\top}\Qt\bv\|_2.\end{equation}
\noindent{\bf (c)}: Now consider the third and the fourth term. Now $\bw_t$ and $\bz_t$ are independent 0-mean Gaussians, hence using Lemma~\ref{lem:rankk_op}, we get: $\|\frac{1}{B}\sum_tU^{\top}\bw_t\bz_t^{\top}\|_2\leq \sigma\sqrt{\frac{C_1k\log(T)}{B}}$. Hence, w.p. $1-2C/T$,{\small
\begin{equation}
  \|\Lambda(\frac{1}{B}\sum_t \bz_t\bw_t^{\top}U)U^{\top}\Qt\bv\|+\|(\frac{1}{B}\sum_t U^{\top}\bw_t\bz_t)\Lambda U^{\top}\Qt\bv\|\leq 2\sigma\sqrt{\frac{C_1k\log(T)}{B}}\|U^{\top}\Qt\bv\|_2. \label{eq:k_lb4}
\end{equation}}
\noindent {\bf (d)}: Finally, we consider the last term in \eqref{eq:k_lb1}. Note that, $(\Lambda\bz_t+U^{\top}\bw_t) \sim N(0, D)$ where $D$ is a diagonal matrix with $D_{ii}=\lambda_i^2+\sigma^2$. Also, $Q^{\top}U_\perp U^{\top}_\perp\bw_t \sim N(0, \sigma^2 I_{(p-k)\times (p-k)})$ and is independent of $(\Lambda\bz_t+U^{\top}\bw_t)$ as $E[Q^{\top}U_\perp U^{\top}_\perp\bw_t\bw_t^{\top}U]=0$; recall that for Gaussian RVs,  covariance is zero iff RVs are independent. Hence, using Lemma~\ref{lem:rankk_op}, w.p. $\geq 1-C/T$: 
\begin{equation}
  \|(\frac{1}{B}\sum_t(\Lambda\bz_t+U^{\top}\bw_t)\bw_t^{\top}U_\perp U_\perp^{\top}\Qt)\bv\|_2\leq \sqrt{1+\sigma^2}\sigma\sqrt{\frac{C_1k\log(T)}{B}}. \label{eq:k_lb5}
\end{equation}
Now, using \eqref{eq:k_lb1}, \eqref{eq:k_lb2}, \eqref{eq:k_lb3}, \eqref{eq:k_lb4}, \eqref{eq:k_lb5} (w.p. $\geq 1-5C/T$)
\begin{multline}
\|U^{\top}\Ftt\Qt\bv\|_2\geq   \|(\Lambda^2+\sigma^2I)U^{\top}\Qt\bv\|_2-\sqrt{\frac{C_1k\log(T)}{B}}\|U^{\top}\Qt\bv\|_2\left((1+\sigma)^2+\frac{\sigma\sqrt{1+\sigma^2}}{\|U^{\top}\Qt\bv\|_2}\right). \label{eq:k_lb6}
\end{multline}
%\todo{This inductive hypothesis, intuitive as it is, requires some explaining.}
Now, $\|U^{\top}\Qt\bv\|_2\geq\sigma_k(U^{\top}\Qt\bv)$. Next, by using the inductive hypothesis (i.e., $\sigma_k(U^T\Qt)\geq \sigma_k(U^TQ_{\tau-1})$, induction step follows as we show that the error decreases at each step) and Lemma~\ref{lem:rankk_init}, we have $\|U^{\top}\Qt\bv\|_2\geq \sigma_k(U^{\top}Q_0)\geq \frac{C}{\sqrt{pk}}$ with probability $\geq 0.99$. 
%Now, by the inductive hypothesis (i.e., $\|U^T\Qt\bv\|_2\geq \|U^{\top}Q_{\tau-1}\bv\|_2\geq $, induction step follows as we show that the error decreases at each step) and Lemma~\ref{lem:rankk_init} (w.p.\ $0.99$), $\|U^{\top}\Qt\bv\|_2\geq \|U^{\top}Q_0\bv\|_2\geq \frac{C}{\sqrt{pk}}$. 

Also,  $\|(\Lambda^2+\sigma^2I)U^{\top}\Qt\bv\|_2\geq (\lambda_k^2+\sigma^2)\|U^{\top}\Qt\bv\|_2. $ Additionally, $\|U^{\top}\Qt\bv\|_2\geq \sqrt{1-\|U_\perp^{\top}\Qt\|_2^2}$. Hence, lemma follows by using these facts with \eqref{eq:k_lb6} and by selecting $B$ as given in Theorem~\ref{thm:rankk}. 
\end{proof}

\section{Proof of Lemma~\ref{lem:rankk_ub}}
\label{app:rankk_ub}
\begin{proof}
	Similar to our proof for Lemma~\ref{lem:rankk_lb}, we separate out the ``error''  or deviation terms in $\|U_\perp^{\top}\Ftt\Qt\|_2$ and bound them using concentration bounds. Now, %Also, note that if the error term is zero than $\|U_\perp^{\top}\Ftt\Qt\|_2=\sigma^2$
\begin{align}\|U_\perp^{\top}\Ftt\Qt\bv\|_2&=\|U_\perp^{\top}(U\Lambda^2U^{\top}+\sigma^2 I + E_\tau)\Qt \bv\|_2\nonumber\\
&\leq \|\sigma^2U_\perp^{\top}\Qt\bv\|_2+\|U_\perp^{\top}E_\tau\Qt\bv\|_2\nonumber\\
&\leq \sigma^2\|U_\perp^{\top}\Qt\bv\|_2+\|E_\tau\|_2,\label{eq:k_ub1}
\end{align} 
where $E_\tau$ is the error matrix representing deviation of the estimate $\Ftt$ from its mean. That is, 
\begin{align}
  E&=\frac{1}{B}\sum_t\bx_t\bx_t^{\top}-U\Lambda^2U^{\top}-\sigma^2I\nonumber\\
&=U\Lambda (\frac{1}{B}\sum_t \bz_t\bz_t^{\top}-I) \Lambda U^{\top} + (\frac{1}{B}\sum_t \bw_t\bw_t^{\top}-\sigma^2 I)\nonumber\\
&+U\Lambda\frac{1}{B}\sum_t\bz_t\bw_t^{\top}+\frac{1}{B}\sum_t \bw_t\bz_t^{\top}\Lambda U. \label{eq:k_ub2}
\end{align}
Note that the above given four terms correspond to similar four terms in \eqref{eq:k_lb1} and hence can be bounded in similar fashion. In particular, the following holds with probability $1-4C/T$: 
\begin{equation}
  \|E\|_2\leq \sqrt{\frac{C_1k\log(T)}{B}}+\sigma^2\sqrt{\frac{C_1p\log(T)}{B}}+2\sigma\sqrt{\frac{C_1p\log(T)}{B}}\leq \lambda_k^2\epsilon/2,\label{eq:k_ub3}
\end{equation}
where the second inequality follows by setting $B$ as required by Theorem~\ref{thm:rankk}. The lemma now follows using \eqref{eq:k_ub1}, \eqref{eq:k_ub2}, \eqref{eq:k_ub3}. 
\end{proof}

\section{Proof of Lemma \ref{lem:rankk_init}}

\begin{proof} 
Using Step 2 of Algorithm~\ref{algo:rank1}: $H=Q_0R_0$. Let  $\bv_k$ be the singular vector of $U^{\top}Q_0$ corresponding to the smallest singular value. Then, 
%Also, let $\bv_k$ be the smallest singular vector of $U^{\top}Q_0$. Then,
\begin{align}
  \sigma_k(U^{\top}Q_0)&=\frac{\|U^{\top}Q_0R_0R_0^{-1}\bv_k\|_2}{\|R_0^{-1}\bv_k\|_2}\|R_0^{-1}\bv_k\|_2\nonumber\\
&\geq \sigma_k(U^{\top}Q_0R_0)\sigma_k(R_0^{-1}). \label{eq:k1_init1}
\end{align}

Now, $\sigma_k(R_0^{-1})=\frac{1}{\|R_0\|_2}=\frac{1}{\|Q_0R_0\|_2}=\frac{1}{\|H\|_2}$. Note that $\|H\|_2$ is the spectral norm of a random matrix with i.i.d. Gaussian entries and hence can be easily bounded using standard results. In particular, using Lemma~\ref{lem:smax_vershynin}, we get: $\|H\|_2\leq C_1\sqrt{p}$ w.p.  $\geq 1-e^{-C_2p}$, where $C_1, C_2>0$ are global constants. 
%Let $H=[\bx_1 \bx_2 \dots \bx_k]=Q_0R_0$. 
%Let $H=[\bx_1 \bx_2 \dots \bx_k]$ be formed by using first $k$ samples as its $k$ columns. %in  wher $x_i$  be generated from a Gaussian distribution, i.e, $H_{ij}\sim N(0,1), \forall i,j$.
% Also, let $H=Q_0R_0$.

By Theorem 1.1 of \cite{rudelson_smallest_2009} (see Lemma~\ref{lem:smin_vershynin}), w.p. $\geq 0.99$, $\sigma_k(U^{\top}Q_0R_0)=\sigma_k(H)\geq C/\sqrt{k}$. The lemma now follows using the above two bounds with \eqref{eq:k1_init1}. 
%Using above given equation with the above given bound on $\sigma_k(R_0^{-1})$, w.p. $1-e^{-C_2p}$, $$\sigma_k(U^{\top}Q_0)\geq \sigma_k(U^{\top}Q_0R_0) \frac{1}{C_1\sqrt{p}}.$$
\end{proof}
%%% Local Variables: 
%%% mode: latex
%%% TeX-master: "onlinepca"
%%% End: 

%%%%%%%%%%%%%%%%%%%%%%%%%%%%%%%%%%%%%%%%%%%%%%%%%
\end{document}